\documentclass[12pt]{article}
\usepackage{natbib, latexsym, amsthm, amsmath, amssymb, color, rotating, multirow}
\usepackage[small,compact]{titlesec}
\usepackage{secdot}
\usepackage{url}
\usepackage{pdfsync}
\usepackage{rotating}
\usepackage[us,12hr]{datetime}
\usepackage{bm}

 \newcommand{\Real}{\mathbb{R}}

 \newcommand{\abs}[1]{\left\vert#1\right\vert}

 \newcommand{\norm}[1]{\left\Vert#1\right\Vert}

\newcommand{\indicator}{{\mathchoice {\rm 1\mskip-5mu l} {\rm 1\mskip-4mu l}{\rm 1\mskip-4.5mu l} {\rm 1\mskip-5mu l}}}

\newcommand{\BF}[1]{\mathbf{#1}}
\newcommand{\CAL}[1]{\mathcal{#1}}

\newcommand{\Regret}{\mbox{Regret}}
\newcommand{\Risk}{\mbox{Risk}}
\newcommand{\E}{\mbox{\sf E}} 

\newcommand{\tr}{\mbox{\sf tr}} 

\newcommand{\singlespacing}{\let\CS=\@currsize\renewcommand{\baselinestretch}{1.0}\tiny\CS}
\newcommand{\doublespacing}{\let\CS=\@currsize\renewcommand{\baselinestretch}{1.5}\tiny\CS}
\newcommand{\realdoublespacing}{\let\CS=\@currsize\renewcommand{\baselinestretch}{2.0}\tiny\CS}
\newcommand{\mydoublespacing}{\let\CS=\@currsize\renewcommand{\baselinestretch}{1.499}\tiny\CS}

\newtheorem{thm}{Theorem}[section]

\newtheorem{corollary}[thm]{Corollary}
\newtheorem{lem}[thm]{Lemma}

\newtheorem{assumption}[]{Assumption}

\newtheorem{example}[thm]{Example}

\newcommand{\bs}{\begin{sffamily}}
\newcommand{\es}{\end{sffamily}}

\newcommand{\ben}{\begin{enumerate}}
\newcommand{\een}{\end{enumerate}}

\newcommand{\be}{\begin{eqnarray}}
\newcommand{\ee}{\end{eqnarray}}
\newcommand{\bex}{\begin{eqnarray*}}
\newcommand{\eex}{\end{eqnarray*}}
\newcommand{\beq}{\begin{equation}}
\newcommand{\eeq}{\end{equation}}

\renewcommand{\baselinestretch}{1.4}

\long\def\pr#1{{\color{black}#1}}
\long\def\jnt#1{{\color{black}#1}}

\begin{document}

\begin{centering}
\title{
Linearly Parameterized Bandits
}

\author{
\begin{tabular}{ccc}
Paat Rusmevichientong & \hspace{0.5in} & John N. Tsitsiklis \\
Cornell University           & \hspace{0.5in} & MIT \\
paatrus@cornell.edu      & \hspace{0.5in} & jnt@mit.edu
\end{tabular}}


\date{January 18, 2010~~@~~11:04pm}

\maketitle
\end{centering}

\abstract{
We consider bandit problems involving a large (possibly infinite) collection of arms, in which
the expected reward of each arm is a linear function of an $r$-dimensional random vector
$\BF Z \in \Real^r$, where $r \geq 2$.  The objective is to minimize the cumulative regret and Bayes
risk.  When the set of arms corresponds to the unit sphere, we prove
that the regret and Bayes risk is of order $\Theta( r \sqrt{T} )$, 
by establishing a lower bound for an arbitrary policy,
and showing that a matching upper bound is obtained through
a policy that alternates between exploration and 
exploitation phases.  The phase-based policy is also shown to be effective 
if the set of arms satisfies a strong convexity condition.
For the case of a general set of arms,
we describe a near-optimal policy
whose regret and Bayes risk admit upper bounds of the form $O(r \sqrt{T} \log^{3/2} T)$.
}


\section{Introduction}

Since its introduction by
\cite{Thompson:1933},  
the multiarmed bandit problem has served as  
an important model for decision making under uncertainty.
Given a set of arms  with unknown reward profiles, the decision maker
must choose a sequence of arms  to maximize the expected total  payoff,
where
the decision in each period may depend on the
previously observed rewards.   The multiarmed bandit problem elegantly captures the tradeoff between
the need to exploit arms with high payoff 
and the incentive to explore previously untried arms for information gathering purposes.

Much of the previous work on the multiarmed bandit problem assumes that the rewards of the arms are 
statistically independent (see, for example, \citet{LaiRobbins:1985} and \citet{Lai:1987}).  This assumption enables us to 
consider each arm separately,
but it leads to policies 
whose regret
scales linearly  with the number of arms.
Most policies that assume independence 
require each arm to be tried at least once, and are 
impractical in settings 
involving 
many arms.
In such settings, we want a policy whose regret is
independent of the number of arms.


When the mean rewards of the arms are assumed to be independent random variables, 
\cite{LaiRobbins:1985} show that the regret under an arbitrary policy must increase linearly
with the number of arms.  However, the assumption of independence is  quite strong in practice.
In many applications, the information obtained from pulling one arm
can change our understanding of other arms.  For instance, in marketing applications,
we expect {\em a priori} that similar products  
should  have similar sales.    By exploiting the correlation
among products/arms, we should be able to obtain a policy whose regret scales more 
favorably than traditional bandit
algorithms that ignore correlation and assume independence.

  
\cite{MersereauRT:2008} propose a simple  model that demonstrates the benefits
of exploiting the underlying structure of the rewards. 
They consider a bandit problem
where the expected reward of each arm is a linear function of an unknown scalar, with a known prior
distribution.
Since the reward of each arm depends on a single random variable, the mean rewards
are  perfectly correlated.  
They prove that, under certain assumptions,  the cumulative Bayes risk over $T$ periods
(defined below) 
under a greedy policy admits an $O \left( \log T \right)$ upper bound, 
independent
of the number of arms. 

In this paper, we extend the model of \cite{MersereauRT:2008} to the setting where the expected
reward of each arm depends linearly on a {\em multivariate} random vector $\BF Z \in \Real^r$. We concentrate on the case where $r \geq 2$, which   
is  fundamentally different from the previous model 
because the mean rewards now depend on more than one random variable,  and thus, they are no longer perfectly correlated.  The bounds on the regret and Bayes risk
and the policies found in  \cite{MersereauRT:2008}  no longer apply.  To give a flavor for the differences,  we will show that, in our model, the cumulative Bayes risk under an arbitrary policy is at least $\Omega( r \sqrt{T})$, which is significantly higher than the upper bound of $O (\log T)$ attainable when $r = 1$.    

The linearly parameterized bandit  is an important model
that has been studied by  many researchers, including 
\citet{GinebraClayton:1995}, \citet{AbeLong:1999}, and
\citet{Auer:2002}.
The results in this paper complement and extend the 
earlier and independent work of
\cite{DHK:2008a}  in a number of directions.   We provide a
detailed 
comparison between our work 
and the existing literature
in Sections \ref{section:literature}~and~\ref{section:contributions}.


\subsection{The Model}
We have a compact set $\CAL U_r \subset \Real^r$ that corresponds to the set of
arms, where $r \geq 2$.  The reward $X_t^{\BF u}$ of playing arm $\BF u \in \CAL U_r$ in period $t$ is given by
\begin{equation}
	X_t^{\BF u} = \BF u^\prime  \BF Z  + 
	W_t^{\BF u}~, 
	\label{eq:model}
\end{equation}
where $\BF u^\prime \BF Z$ is the inner product between the vector 
$\BF u  \in \CAL U_r$ and the random vector $\BF Z  \in \Real^r$.   
We assume that
the random variables
$W_t^{\BF u}$ are independent of each other and of $\BF Z$. Moreover, for each $\BF u \in \CAL U_r$, the random variables
$\left\{ W_t^{\BF u} : t \geq 1 \right\}$ are identically distributed, with $\E \left[ W_t^{\BF u} \right] = 0$ for all $t$ and $\BF u$.
We allow 
the error
random variables $W^{\BF u}_t$ to have unbounded support,
provided that their moment generating functions satisfy certain conditions
(given in Assumption~\ref{assumption:basic}).
Each vector $\BF u \in \CAL U_r$ simultaneously represents an arm and determines the expected 
reward of that arm.
So, when it is clear from the context, we will interchangeably refer to a $\BF u \in \CAL U_r$
as either a vector or an arm.   

Let us introduce the following conventions and notation that will be used throughout the paper.
We denote vectors and matrices in bold.
All vectors are column vectors. 
For any vector $\BF v \in \Real^r$, its transpose is denoted by $\BF v^\prime$, 
and is always a row vector. Let $\BF 0$ denote the zero vector, and
for $k = 1, \ldots, r$, let $\BF e_k = (0, \ldots, 1, \ldots, 0)$ denote
the standard unit vector in $\Real^r$, with a $1$ in the $k^{th}$ component
and a $0$ elsewhere. Also, let $\BF I_k$ denote the $k \times k$ identity matrix.
We let $\BF A^\prime$  and $\det (\BF A)$ 
denote
the transpose and determinant 
of $\BF A$, respectively. 
If $\BF A$ is a symmetric positive semidefinite matrix, then
$\lambda_{\min}(\BF A)$ and $\lambda_{\max}(\BF A)$ denote the smallest
and the largest eigenvalues of $\BF A$, respectively. 
We  use $\BF A^{1/2}$
to denote its symmetric nonnegative definite square root, so that $\BF A^{1/2} \BF A^{1/2}  = \BF A$.
If $\BF A$ is also positive definite, we let $\BF A^{ -1/2} = \left( \BF A^{-1} \right)^{1/2}$.
For any vector $\BF v$,
$\left\| \BF v \right\| = \sqrt{\BF v^\prime \BF v}$ denotes the standard Euclidean norm,
and for any positive definite matrix $\BF A$,
$\left\| \BF v \right\|_{\BF A} = \sqrt{ \BF v^\prime 
\BF A \BF v}$
denotes a corresponding weighted norm.  
For any two symmetric positive semidefinite
matrices $\BF A$ and $\BF B$, we write $\BF A \leq \BF B$ if the matrix $\BF B - \BF A$ is
positive semidefinite.
Also, all logarithms $\log(\cdot)$ in the paper denote the natural log, 
with base $e$.
 A random variable is denoted by an uppercase letter while its realized values are
 denoted in lowercase. 

For any $t \geq 1$, let $\CAL {H}_{t-1}$ denote the set of possible histories until the end 
of period $t-1$.  A policy ${\BF \psi} = \left( \psi_1, \psi_2, \ldots \right)$ is a sequence of functions
such that $\psi_t : \CAL{H}_{t-1} \to \CAL U_r$ selects an arm in period $t$ based on the history
until the end of period $t-1$.  For any policy $\BF\psi$ and $ \BF z \in \Real^r$, the $T$-period
cumulative {\em regret} under $\BF\psi$ given $\BF Z = \BF z $, denoted by $\Regret( \BF z, T, \BF\psi)$,
is defined by
$$
	\Regret( \BF z, T, \BF\psi)  = \sum_{t=1}^T \E  \left[ \max_{\BF v \,\in\, \CAL U_r} 
		\BF v^\prime \BF z -  \BF U_{t}^\prime \BF z ~\Big|~ \BF Z = \BF z \right]~,
$$
where for any $t \geq 1$, $\BF U_t \in \CAL U_r$ is the arm chosen under
$\BF \psi$ in period $t$.  
Since $\CAL U_r$ is compact, $\max_{\BF v \,\in\, \CAL U_r} \BF v^\prime \BF z$ is 
well defined for all $\BF z$. The $T$-period cumulative {\em Bayes risk} under $\BF\psi$ is defined by
$$
	\Risk(T, \BF\psi) = \E \left[ \Regret (\BF Z, T, \BF\psi ) \right]~,
$$
where the expectation is taken with respect to the prior distribution of $\BF Z$.
We aim to develop a policy that minimizes the cumulative regret and Bayes risk.    We note that
minimizing the $T$-period cumulative Bayes risk is equivalent to maximizing the expected total
reward over $T$ periods. 

To facilitate
exposition, when we discuss a particular policy, we will drop the superscript and write $X_t$ and $W_t$ to denote $X_t^{\BF U_t}$ and $W_t^{\BF U_t}$, respectively, where $\BF U_t$ is the arm chosen by the policy in period $t$.   With this convention, the reward obtained in period $t$ under a particular policy is simply $X_t = \BF U_t^\prime \BF Z  + W_t$.

\subsection{Potential Applications}

Although our paper focuses on a theoretical analysis, we mention briefly potential applications
to problems in marketing and revenue management.  Suppose we have $m$
arms indexed by $\CAL U_r = \left\{ \BF u_1, \BF u_2, \ldots, \BF u_m \right\} \subset \Real^r$.  For 
$k=1,\ldots, r$, let $\bm{\phi}_k = \left( u_{1,k}, u_{2,k}, \ldots, u_{m,k} \right) \in \Real^m$
denote an $m$-dimensional column vector consisting of the $k^{th}$ coordinates of the
different vectors $\BF u_\ell$.   
Let $\bm{\mu} = \left( \mu_1, \ldots, \mu_m \right)$ be the column vector
consisting of expected rewards, where $\mu_\ell$ denotes the expected reward of arm $\BF u_\ell$.
Under our formulation, the vector  
$\bm{\mu}$ lies in an $r$-dimensional
subspace spanned by the vectors $\bm{\phi}_1, \ldots, \bm{\phi}_r$, that is,
$\bm{\mu} = \sum_{k=1}^r Z_{k} \bm{\phi}_k$, where $\BF Z = \left( Z_{1}, \ldots, Z_{r} \right)$.
If each arm corresponds to a product  to be offered to a customer, we can then interpret the vector $\bm{\phi}_k$ as a feature vector or basis function, representing a particular characteristic of the products
such as  price or popularity.  
We can then interpret the random variables $Z_{1}, \ldots ,Z_{r}$ as regression coefficients, obtained 
from approximating the vector of expected rewards  using the basis functions $\bm{\phi}_1, \ldots, \bm{\phi}_r$,
or more intuitively, as the weights associated with the different characteristics.
  Given a prior on the coefficients
$Z_{k}$, our goal is to choose a sequence of products that gives the maximum expected total reward.
This representation suggests
that our model might be applicable to problems where we can approximate high-dimensional vectors
using a linear combination of a few basis functions, an    
approach that has been  successfully applied to 
high-dimensional dynamic programming problems
(see \cite{BertsekasTsitsiklis:1996} for an overview). 

\subsection{Related Literature}
\label{section:literature}

The multiarmed bandit literature
can be divided into two streams, depending on the objective function criteria: maximizing
the total discounted reward over an infinite horizon or minimizing the cumulative
regret and Bayes risk over a finite horizon.  Our paper focuses exclusively on 
the second criterion. 
Much of the work in this area focuses on understanding the rate with which the regret and risk 
under various policies increase over time.   In their pioneering work,
\citet{LaiRobbins:1985} 
establish an asymptotic lower bound 
of $\Omega( m \log T)$ for the $T$-period cumulative regret for bandit problems with $m$ independent
arms \pr{whose mean rewards are ``well-separated\jnt{,}'' where the difference between the expected
reward of the best and second best arms is fixed and bounded away from zero.}
They further demonstrate a policy whose regret 
asymptotically matches the lower bound.   \pr{In contrast, our paper focuses
on problems where the number of arms is large (possibly infinite), and \jnt{where}
the gap between the maximum expected reward and the expected reward of the second best arm can
be arbitrarily small.}
\citet{Lai:1987} extends these results to a Bayesian setting, with 
a prior distribution on the reward characteristics of each arm. \pr{He 
shows
that when we have $m$ arms, the $T$-period cumulative Bayes risk 
is of order $\Theta( m \log^2 T)$, when the prior distribution has a continuous density
function satisfying certain properties (see Theorem 3 in \citealp{Lai:1987}).}
  Subsequent papers along this line
include \citet{Agrawal:1989a}, \citet{Agrawal:1995}, and \citet{AuerCF:2002}.


There has
been relatively little research, however, on policies that exploit the dependence among the arms.
\citet{Thompson:1933} allows for correlation
among arms in his initial formulation, though he only analyzes a
special case involving independent arms.  \citet{Robbins:1952} formulates a continuum-armed 
bandit regression problem, but
does not provide an analysis of the regret or risk.
\citet{BerryFristedt:1985} allow for dependence among arms in their formulation in Chapter 2, but mostly focus on the case of independent arms.
\citet{Feldman:1962} and \citet{Keener:1985} consider two-armed bandit
problems with two hidden states, where the rewards of each arm
depend on the underlying state that prevails.
\citet{PressmanSonin:1990} formulate  a general multiarmed bandit problem with an arbitrary number of hidden states, and provide a detailed analysis for the case of two hidden states.
\citet{PandeyChakrabartiAgarwal:2007} study bandit
problems where the dependence of the arm rewards is represented by a hierarchical model. 

A somewhat related literature on bandits with dependent arms is the recent work 
by \citet{WangKulkarniPoor:2005b, WangKulkarniPoor:2005a} and \citet{GoldenshlugerZeevi:2007,
GoldenshlugerZeevi:2008} 
who consider bandit problems with two arms, where the expected reward of each arm 
depends on an exogenous variable that represents side information.  
These models, however, differ from ours because
they assume that the side information variables are independent and identically distributed 
over time,  and moreover, these variables are {\em perfectly observed before} we choose which 
arm to play.  
In contrast, we assume that the underlying random vector $\BF Z$
is unknown and fixed over time, to be estimated based on past
rewards and decisions.  

\pr{Our formulation can be viewed as a sequential method for maximizing a linear function
based \jnt{on} noisy observations of the function values, and  it is thus closely related
to the field of stochastic approximation, which was developed by
\cite{RobbinsMonro:1951} and \cite{KieferW:1952}.  We do not provide a comprehensive review of the literature here; interested
readers are referred to an excellent survey by \cite{Lai:2003}.   In stochastic
approximation, we wish to find an adaptive sequence $\left\{ \BF U_t  \in \Real^r : t \geq 1 \right\}$
that converges to a maximizer $\BF u^*$ of a target function,
and the focus is on establishing  the rate at which the mean squared error $\E \left[ \norm{ \BF U_t - \BF u^*}^2 \right]$ converges to zero (see, for example, \citealp{Blum:1952} and \citealp{CicekBZ:2009}).
In contrast, our cumulative regret and Bayes risk criteria take into account the cost associated
with each observation.  The different performance measures used in our formulation lead to entirely different policies and performance characteristics.}

Our model generalizes the ``response surface bandits''  introduced by
\citet{GinebraClayton:1995}, who assume a normal prior on $\BF Z$
and provide a simple tunable heuristic, without
any analysis on the regret or risk.  
\citet{AbeLong:1999}, \citet{Auer:2002}, and \cite{DHK:2008a}
all consider a special case of our model where 
the random vector $\BF Z$ 
and the error random variables $W^{\BF u}_t$ are bounded almost surely, and 
with the exception of the last paper,
focus on the regret criterion.    \citet{AbeLong:1999}
demonstrate a class of bandits  where 
the dimension $r$ is at least $\Omega(\sqrt{T})$, and show that the $T$-period
regret under an arbitrary policy must be at least
$\Omega \left( T^{3/4} \right)$.  
\citet{Auer:2002} describes an algorithm based on least squares
estimation and confidence bounds, and establishes an $O\left( \sqrt{r} \sqrt{ T}  \log^{3/2} \left( T \abs{\CAL U_r} \right) \right)$ upper bound on the regret, for  the case of finitely many arms.
\cite{DHK:2008a} show that the policy of \citet{Auer:2002} can
be extended to problems having an arbitrary compact
set of arms, and also make use of a 
barycentric spanner. 
They establish an $O ( r \sqrt{T} \log^{3/2} T)$
upper bound on the regret, and discuss a variation of the policy that is more computationally
tractable (at the expense of  higher regret).
\citet{DHK:2008a} also establish an $\Omega( r \sqrt{T} )$
lower bound on the Bayes risk
when the set of
arms is the 
Cartesian product of circles\footnote{The original lower bound (Theorem~3 on page 360 of \citealp{DHK:2008a}) was not entirely correct; a correct version was provided later, in \citet{DHK:2008b}.}.
However, this leaves a $O(\log^{3/2}T)$ gap from the upper bound, leaving open the question of the exact order of regret and risk.

\subsection{Contributions and Organizations}
\label{section:contributions}

One of our contributions is 
proving that the regret and Bayes risk for a broad class of linearly parameterized
bandits is of order $\Theta ( r \sqrt{T} )$.
In Section \ref{section:lower-bound},
we establish an 
$\Omega( r \sqrt{T})$ lower bound
for an arbitrary policy,
when the set of arms is the unit sphere in $\Real^r$.
Then, in Section \ref{section:matching-upper-bound}, 
we show
that 
a matching $O(r \sqrt{T})$ upper bound 
can be achieved through a phase-based policy that alternates between exploration and
exploitation phases. 
To the best of our knowledge, this is the first  
result that establishes
matching upper and lower bounds for a class of linearly parameterized bandits. 
Table \ref{table:bound-summary} summarizes 
our results and provides a comparison with  the results in \citet{MersereauRT:2008} for the case $r=1$. 
In the ensuing discussion of the bounds, we focus on the main parameters, $r$ and $T$, with more precise statements given in the theorems.

Although we obtain the same lower bound of $\Omega( r \sqrt{T} )$,
our example and proof techniques are very different from \citet{DHK:2008a}.
We consider the unit sphere, with
a multivariate normal prior on $\BF Z$, and
standard normal errors.
The analysis in Section~\ref{section:lower-bound} also illuminates the behavior of the least mean squares estimator in this setting, and we believe that it 
provides an approach that can be used to address more general classes
of linear estimation and adaptive control problems.

We  also prove that 
the phase-based policy remains effective (that is, admits an
$O(r \sqrt{T})$ upper bound) for a broad class of bandit problems in which the set 
of arms is strongly convex\footnote{One can show that the Cartesian product of circles is not strongly convex, and thus, our
phase-based policy cannot be applied to give the matching upper
bound for the example used in \citet{DHK:2008a}.} (defined in Section~\ref{section:matching-upper-bound}).
To our knowledge, this is the first result that establishes the connection
between a geometrical property (strong convexity) of the underlying set of arms and 
the effectiveness of separating exploration from exploitation. We suspect that
strong convexity may have similar implications for other types of bandit and learning problems.



When the set of arms is an arbitrary compact set, the separation of exploration
and exploitation may not be effective, and we consider in Section \ref{section:active-exploration} an active exploration policy based on least squares estimation and confidence regions. 
We prove that
the regret and risk under this policy are bounded above by $O( r \sqrt{T} \log^{3/2} T )$, which is within
a logarithmic factor of the lower bound. 
Our policy is closely related to the  one considered in \citet{Auer:2002} and further analyzed in \citet{DHK:2008a}, with differences in a number of respects. First, 
our model allows 
the random vector $\BF Z$ and the errors $W^{\BF u}_t$ to have unbounded support, which requires a somewhat more complicated analysis. Second, our
policy is an ``anytime'' policy, in the sense that the policy does not depend on the time horizon $T$ of interest.  In contrast, 
the policies of \citet{Auer:2002} and \citet{DHK:2008a} involve
a certain parameter $\delta$ whose value must be 
set in advance as a function of the time horizon $T$ in order to obtain the 
 $O \left( r \sqrt{T} \log^{3/2} T \right)$ regret bound.

\begin{table}[h]
\centering  \fontsize{9}{9}\selectfont{{
\begin{tabular}{|c|c|c|c|c|c|}
\hline
	&	  &  \multicolumn{2}{c|}{ } & \multicolumn{2}{|c|}{ }  \\
       &        & \multicolumn{2}{|c|}{$T$-period Cumulative Regret} & \multicolumn{2}{|c|}{$T$-period Cumulative  Bayes Risk}  \\
Dimension &  Set of  	  &  \multicolumn{2}{|c|}{ } & \multicolumn{2}{|c|}{ }  \\
\cline{3-6}
($r$)	&	 Arms ($\CAL U_r$) &  &  &  &\\
	&	   & Lower Bound & Upper Bound & Lower Bound & Upper Bound\\
	&	  &  &  &  &\\
\hline
\hline
	     		   &	     &   &   &   &  \\
$r=1$	   	   &	Any Compact Set      &  $\Omega \left( \sqrt{T} \right)$ & $O \left( \sqrt{T} \right)$  &    $\Omega \left( \log T \right)$    &   $O \left( \log T \right)$  \\
   & ({\fontsize{7}{7}\selectfont{Mersereau et al., 2008}})  &   &   &   &  \\
\hline
\hline
				&	 			&   &   &   &  \\
	         		&	Unit Sphere &  $\Omega \left( r \sqrt{T} \right)$ & $O \left( r  \sqrt{T} \right)$  &    $\Omega \left( r  \sqrt{T} \right)$    &   $O \left( r  \sqrt{T} \right)$  \\
$r \geq 2$	&  (Sections \ref{section:lower-bound} and \ref{section:matching-upper-bound})  &   &   &   &  \\
\cline{2-6}
({\fontsize{8}{8}\selectfont{This Paper}})				&	 &   &   &   &  \\
	   & Any Compact Set	 &  $\Omega\left( r\sqrt{T} \right)$ & $O\left( r\sqrt{T} \log^{3/2} T\right)$  &    $\Omega\left( r \sqrt{T}\right)$    &   $O\left(r \sqrt{T} \log^{3/2} T\right)$  \\
	   & (Section \ref{section:active-exploration}) &   &   &   &  \\
\hline
\end{tabular}}}
\caption{{\fontsize{10}{10}\selectfont{Regret and risk bounds for various values of $r$ and for different collections of arms.}}}
\label{table:bound-summary}
\end{table}

We finally comment on the case where the set of arms is finite and fixed.
We show that the regret and risk under our active exploration policy increase gracefully with time, as $\log T$ and $\log^2 T$, respectively.  These results show that
our policy is within a constant factor of the asymptotic lower bounds established by \citet{LaiRobbins:1985} and \citet{Lai:1987}.  In contrast, 
for the policies of \citet{Auer:2002}
and \citet{DHK:2008a}, the available regret upper bounds 
grow over time as $\sqrt{T} \log^{3/2} T$ and $\log^3 T$, respectively. 

\pr{We note that the bounds on the cumulative Bayes risk given in Table \ref{table:bound-summary} hold under certain assumptions on the prior distribution of the random vector $\BF Z$.  For $r=1$, $\BF Z$ is assumed to be a continuous random variable with a bounded density function (Theorem 3.2 in \citealp{MersereauRT:2008}). When the collection of arms is a unit sphere with $r\geq2$, we require that both $\E \left[ \norm{\BF Z} \right]$ and $\E \left[ 1 / \norm{\BF Z} \right]$ are bounded (see Theorems  \ref{thm:lower-bound} and \ref{theorem:greedy-regret}, and Lemma \ref{lemma:gaussian-has-small-mass}).   For general compact sets of arms where our risk bound is not tight, we only require that $\norm{\BF Z}$ has a  bounded expectation.}

\section{Lower Bounds}
\label{section:lower-bound}

In this section, we establish $\Omega( r \sqrt{T} )$ lower bounds on the regret and risk under an arbitrary policy when the set of arms is the unit sphere. This result is stated in the following theorem\footnote{The result of Theorem \ref{thm:lower-bound} easily extends to the case where the covariance matrix is $\BF I_r$, rather than $\BF I_r /r$. The proof is essentially the same.}   

\begin{thm}[Lower Bounds] \label{thm:lower-bound}
Consider a bandit problem where the set of arms is the unit sphere in $\Real^r$, and $W^{\BF u}_t$ has a standard normal distribution with mean zero and variance one for all $t$ and $\BF u$.  If $\BF Z$ has a multivariate normal  distribution with mean $\BF 0$ and covariance matrix $\BF I_r / r$, then for all policies $\psi$ and every $T \geq r^2$,
$$
	{\rm Risk} \left( T, \psi \right) \geq 0.006 \,r\, \sqrt{T}~.
$$
Consequently, for any policy $\psi$ and $T \geq r^2$, there exists $\BF z \in \Real^r$ such that
$$
	{\rm Regret} \left( \BF z, T, \psi \right) \geq 0.006 \, r \, \sqrt{T}~.
$$
\end{thm}


It suffices to establish the lower bound on the Bayes risk because  the regret bound  
follows immediately.   
Throughout this section, we assume that 
$\CAL U_r = \left\{ \BF u \in \Real^r : \norm{\BF u} = 1 \right\}$.
We fix an arbitrary policy $\BF \psi$ and for any $t \geq 1$, we let 
$\BF H_t= \left( {\BF U}_1, X_1,  {\BF U}_2, X_2, 
\ldots, {\BF U}_t, X_t \right)$ be the history up to time $t$.
We also let $\widehat{\BF Z}_t$ denote the least mean squares
estimator of $\BF Z$ given the history $\BF H_t$, that is,
$$
	\widehat{\BF Z}_t = \E \left[ \BF Z ~\big|~  \BF H_t \right]~.
$$
Let $\BF S_t^1, \ldots, \BF S_t^{r-1}$ denote a collection of 
orthogonal unit vectors that are also orthogonal to 
$\widehat{\BF Z}_t$.  Note that $\widehat{\BF Z}_t$ and $\BF S_t^1, \ldots, \BF S_t^{r-1}$
are functions of $\BF H_t$.  

Since $\CAL U_r$ is the unit sphere,  $\max_{\BF u \in \CAL U_r} \BF u^\prime \BF z = \left( \BF z^\prime \BF z \right) / \norm{\BF z} = \norm{\BF z}$, for all $\BF z \in \Real^r$.  Thus, 
the risk in period $t$ is given by $\E \left[\, \norm{\BF Z} - \BF U_t^\prime \BF Z \, \right]$.
The following lemma establishes a lower
bound on the  cumulative risk  in terms of
the estimator error variance and
the total
amount of exploration along the directions $\BF S^1_T, \ldots, \BF S^{r-1}_T$.

\begin{lem}[Risk Decomposition] \label{lemma:regret-exploration-estimation}
For any $T \geq 1$, 
$$	
	{\rm Risk} \left( T, \BF \psi \right)
		~\geq~  \frac{1}{2} \sum_{k=1}^{r-1}   \E \left[  \norm{\BF Z}   \sum_{t=1}^T  \left(\BF U_t^\prime    {\BF S}_T^k  \right)^2 	 		+  \frac{T}{\norm{\BF Z}}   
					\left\{ \left( \BF Z - \widehat{\BF Z}_T \right)^\prime {\BF S}_T^k \right\}^2
		\right]~.
$$
\end{lem}
\begin{proof} Using the fact that for any two unit vectors $\BF w$ and $\BF v$, $ 1 - \BF w^\prime \BF v = \norm{\BF w - \BF v}^2/2$, 
the instantaneous regret in period $t$ satisfies
$$
	\norm{\BF Z} - \BF U_t^\prime \BF Z 
	= \norm{\BF Z} \left( 1 - \BF U_t^\prime \frac{\BF Z}{\norm{\BF Z}} \right)
	= \frac{\norm{\BF Z}}{2} \norm{ \BF U_t -  \frac{\BF Z}{\norm{\BF Z}} }^2 
	\geq \frac{\norm{\BF Z}}{2}  
		\sum_{k=1}^{r-1} \left\{ 
	\left( \BF U_t -  \frac{\BF Z}{\norm{\BF Z}} \right)^\prime \BF S^{k}_T \right\}^2~,
$$
where the inequality follows from the fact that 
$\BF S_T^1,\ldots, \BF S_T^{r-1}$ are orthogonal unit vectors.  Therefore,
the cumulative conditional risk satisfies
\bex
2 \, \sum_{t=1}^T \E \left[ \norm{\BF Z} - \BF U_t^\prime \BF Z ~\Big|~ \BF H_T \right] 
		&\geq& \sum_{t=1}^T \E \left[
		\norm{\BF Z} 	
		\sum_{k=1}^{r-1} \left\{ 
	\left( \BF U_t -  \frac{\BF Z}{\norm{\BF Z}} \right)^\prime \BF S^{k}_T \right\}^2
	~\Big|~
		\BF H_T \right]\\
		&=& 
		\sum_{t=1}^T \sum_{k=1}^{r-1} \E \left[
		\norm{\BF Z} 	
	 \left\{ 
	\left( \BF U_t -  \frac{\BF Z}{\norm{\BF Z}} \right)^\prime \BF S^{k}_T \right\}^2
	~\Big|~
		\BF H_T \right]\\
	&=&  \sum_{t=1}^T \sum_{k=1}^{r-1} \E \left[
		\norm{\BF Z}  \left( \BF U_t^\prime   \BF S_T^k  \right)^2
		- 2 \left( \BF U_t^\prime   \BF S_T^k \right) \left( \BF Z^\prime  \BF S_T^k \right)
		+ \frac{ \left( \BF Z^\prime \BF S_T^k \right)^2}{\norm{\BF Z}}  ~\Big|~
		\BF H_T \right]~,
\eex
with probability one. From the definition of $\BF S^k_T$, we have $\widehat{\BF Z}_T^\prime \BF S_T^k = 0$ for $k = 1 \ldots, r-1$.  Therefore, for $t\leq T$,
$$
\E \left[  \left( \BF U_t^\prime   \BF S_T^k \right) \left( \BF Z^\prime  \BF S_T^k \right)
				~\Big|~ \BF H_T \right]
	=    \left( \BF U_t^\prime   \BF S_T^k \right) \E \left[ \BF Z^\prime 
				~\Big|~ \BF H_T \right]  \BF S_T^k 
	=   \left( \BF U_t^\prime   \BF S_T^k \right)\widehat{\BF Z}_T^\prime \BF S_T^k = 0~,
$$
which eliminates the middle term in the summand above.  Furthermore, 
we see that $\BF Z^\prime {\BF S}_T^k =  \left( \BF Z - \widehat{\BF Z}_T \right)^\prime {\BF S}_T^k$ for all $k$.  Thus, with probability one,
\bex
\sum_{t=1}^T \E \left[ \norm{\BF Z} - \BF U_t^\prime \BF Z ~\Big|~ \BF H_T \right] 
	&\geq&  \frac{1}{2} \sum_{k=1}^{r-1} \E \left[
 \norm{\BF Z}   \sum_{t=1}^T  \left(\BF U_t^\prime    {\BF S}_T^k  \right)^2 	 		+  \frac{T}{\norm{\BF Z}}   
					\left\{ \left( \BF Z - \widehat{\BF Z}_T \right)^\prime {\BF S}_T^k \right\}^2 ~\Big|~
		\BF H_T \right]~,
\eex
and the desired result follows by taking the expectation of both sides.
\end{proof}

Since 
${\BF S}_T^k$ is orthogonal to 
$\widehat{\BF Z}_T$,
we
can interpret $\sum_{t=1}^T \left(\BF U_t^\prime    {\BF S}^k_T  \right)^2$
and  $\left\{ \left( \BF Z - \widehat{\BF Z}_T \right)^\prime {\BF S}_T^k \right\}^2$
as the total amount of
exploration over $T$ periods
and the squared estimation error, respectively, in the direction $\BF S^k_T$.
Thus, 
Lemma \ref{lemma:regret-exploration-estimation} tells us that the  
cumulative risk is bounded below by the sum of the squared estimation error and the 
total amount of exploration in the past $T$ periods. This result suggests an approach for establishing a lower bound on the risk.  If the 
amount of exploration
$\sum_{t=1}^T \left(\BF U_t^\prime    {\BF S}_T^k \right)^2$ is large, then the risk will be large.
On the other hand, if the amount of exploration is small, we expect significant
estimation errors, which in turn imply large risk.  This intuition is made precise in Lemma \ref{lemma:estimation-and-exploration}, which relates the squared estimation
error and the amount of exploration.   

\begin{lem}[Little Exploration Implies Large Estimation Errors] \label{lemma:estimation-and-exploration}
For any $k$ and $T \geq 1$,
$$
	\E \left[  \left\{ \left( \BF Z - \widehat{\BF Z}_T \right)^\prime {\BF S}_T^k \right\}^2
	 			~\bigg|~ \BF H_T \right]
		~\geq~ \frac{1}{r + \sum_{t=1}^T \left( \BF U_t^\prime \BF S^k_T \right)^2  }~,
$$
with probability one.
\end{lem}
\begin{proof}
Let $\BF Q_T = \widehat{\BF Z}_T \big/ \|\widehat{\BF Z}_T \|$.  For any $t$, we have that
$\BF U_t = \sum_{k=1}^{r-1} \left( \BF U_t^\prime \BF S_T^k \right) \BF S_T^k  + \left( \BF U_t^\prime \BF Q_T \right) \BF Q_T $.  Let 
$$
\BF V = \left[  \BF S_T^1 \;\;  \BF S_T^2 \;\; \cdots \;\; \BF S_T^{r-1} \;\; \BF Q_T \right]
$$
be an $r \times r$ orthonormal matrix whose columns are the vectors $\BF S_T^1, \ldots, \BF S_T^{r-1}$, and $\BF Q_T$, respectively.
 Then, it is easy to verify that
 $$
 	\sum_{t=1}^T \BF U_t \BF U_t^\prime  =  \BF V \BF A\BF V^\prime~,
 $$
where 
$
		\BF A
		=
		\left(
		\begin{array}{cc}
		\BF \Sigma & \BF c \\
		\BF c^\prime & a
		\end{array}
		\right)~,
$
is an $r \times r$ matrix,
with $a = \BF Q_T^\prime \left(\sum_{t=1}^T  \BF U_t \BF U_t^\prime  \right) \BF Q_T$, $\BF c$ is an $(r-1)$-dimensional
column vector,
and where $\BF \Sigma$ is an $(r-1) \times (r-1)$ matrix with 
$
		\BF \Sigma_{k\ell} = 		\left( \BF S^k_T \right)^\prime \left( \sum_{t=1}^T \BF U_t \BF U_t^\prime \right) \BF S_T^{\ell} = \sum_{t=1}^T  \left( \BF U_t^\prime  \BF S_T^{k} \right) 
		\left( \BF U_t^\prime  \BF S_T^{\ell} \right)
$ for $k,\ell = 1, \ldots, r-1$,

Since $\BF Z$ has a multivariate normal prior distribution with covariance matrix $\BF I_r / r$, it is a standard result (use, for example, Corollary E.3.5 in Appendix E in \citealp{Bertsekas1:1995}) that 
$$
	\E \left[  \left( \BF Z - \widehat{\BF Z}_T \right) \left( \BF Z - \widehat{\BF Z}_T \right)^\prime
	 			~\bigg|~ \BF H_T \right]
				= \left( r \, \BF I_r +  \sum_{t=1}^T \BF U_t \BF U_t^\prime \right)^{-1}
				= \BF V \left( r \, \BF I_r + \BF A \right)^{-1} \BF V^\prime~.
$$
Since $\BF S^k_T$ is a function of $\BF H_T$ and
$\BF V^\prime \BF S^k_T = \BF e_k$, we have, for $k\leq r-1$, that
\bex
	\E \left[  \left\{ \left( \BF Z - \widehat{\BF Z}_T \right)^\prime {\BF S}_T^k \right\}^2
	 			~\bigg|~ \BF H_T \right]
				&=& \left( \BF V^\prime \BF S_T^k \right)^\prime  \left( r \BF I_r + \BF A \right)^{-1}
 \left( \BF V^\prime \BF S_T^k \right) 
 				~=~ \left[ \left( r \,\BF I_{r}  + \BF A \right)^{-1} \right]_{kk}\\
 &\geq& \frac{1}{ \left( r \, \BF I_{r}  + \BF A \right)_{kk}} = \frac{1}{r + \sum_{t=1}^T \left(\BF U_t^\prime \BF S^k_T \right)^2}~,
\eex
where the inequality follows from Fiedler's Inequality (see, for example, Theorem 2.1 in \citealp{FiedlerP:1997}), and the final
equality follows from the definition of $\BF A$.
\end{proof}

The next lemma gives a  lower bound on the probability that
$\BF Z$ is bounded away from the origin.  The proof follows from simple
calculations involving normal densities, and the details are given in Appendix  \ref{section:proof-gaussian-lower-bound}. 

\begin{lem} \label{lemma:gaussian-lower-bound} For any $\theta \leq 1/2$ and $\beta > 0$,
$\Pr \left\{ \theta \leq \norm{ \BF Z} \leq \beta \right\} \geq 1 - 4 \theta^2 - \frac{1}{\beta^2}~$.
\end{lem}

The last lemma establishes a lower bound on the sum of the total amount of exploration and the squared estimation error, which is also the minimum cumulative Bayes risk along the direction $\BF S^k_T$ by Lemma \ref{lemma:regret-exploration-estimation}.

\begin{lem}[Minimum Directional Risk] \label{lemma:sum-of-exploration-and-estimation} For {$k=1,\ldots,r-1$,} and $T \geq r^2$,
$$
 \E \left[  \norm{\BF Z}   \sum_{t=1}^T  \left(\BF U_t^\prime    {\BF S}_T^k  \right)^2 	 		+  \frac{T}{\norm{\BF Z}}   
					\left\{ \left( \BF Z - \widehat{\BF Z}_T \right)^\prime {\BF S}_T^k \right\}^2
	 \right]
	\geq 0.027 \sqrt{T}~.			
$$
\end{lem}
We note that if $\norm{\BF Z}$ were a constant, rather than a random variable,  Lemma \ref{lemma:sum-of-exploration-and-estimation} would follow immediately. Hence, most of the work in the proof below involves constraining
$\norm{\BF Z}$ to a certain range $[\theta,\beta]$.
\begin{proof} Consider an arbitrary $k$, and let
$
	\Xi  =  \sum_{t=1}^T  \left(\BF U_t^\prime    {\BF S}_T^k \right)^2 
$,
$	\Gamma =  \left\{ \left( \BF Z - \widehat{\BF Z}_T \right)^\prime {\BF S}_T^k \right\}^2
$.
Our proof will make use of positive constants $\theta$, $\beta$, and $\eta$,
whose values will be chosen later.  Note that
\bex
 \E \left[  \norm{\BF Z}  \Xi  + 
					\frac{T \,   \Gamma }{\norm{\BF Z}}
	 			~\Big|~ \BF H_T \right]
	&\geq&   \E \left[ \left( \norm{\BF Z} \Xi  + 
					\frac{T \,  \Gamma }{ \norm{\BF Z} } \right)
					\indicator_{\left\{ \theta  ~\leq~ \norm{\BF Z} ~\leq~ \beta \right\}}
								\indicator_{\left\{ \Xi  ~\geq~ \sqrt{T} \right\}}
	 			~\Big|~ \BF H_T \right]\\
	&& ~+~   \E \left[   \left( \norm{\BF Z} \Xi  + 
					\frac{T \, \Gamma}{ \norm{\BF Z} } \right)
					\indicator_{\left\{ \theta  ~\leq~ \norm{\BF Z} ~\leq~ \beta \right\}}
								\indicator_{\left\{ \Xi  ~<~ \sqrt{T} \right\}}
	 			~\Big|~ \BF H_T  \right]\\
	&\geq&  \theta \sqrt{T} \,  \indicator_{\left\{ \Xi  ~\geq~ \sqrt{T } \right\}} \E \left[ 					\indicator_{\left\{ \theta  ~\leq~ \norm{\BF Z} ~\leq~ \beta  \right\}}
	 			~\big|~ \BF H_T \right]\\
	&& ~+~   \frac{T}{\beta} \,\indicator_{\left\{  \Xi  ~<~ \sqrt{T} \right\}} \E \left[    
					  \Gamma  \,
					\indicator_{\left\{ \theta  ~\leq~ \norm{\BF Z} ~\leq~ \beta \right\}}
	 			~\big|~ \BF H_T  \right]~,
\eex
where we use the fact that $\Xi$ is a function of $\BF H_T$ in the final inequality.
We will now lower bound the last term on  the right hand side of the above inequality.  Let $\Theta = \E \left[ \Gamma ~\big|~ \BF H_T \right]$.
Since $\Theta$ is  a function of
$\BF H_T$, 
\bex
\frac{T}{\beta} \,\indicator_{\left\{  \Xi  ~<~ \sqrt{T} \right\}}
\E \left[   \Gamma  \,
					\indicator_{\left\{ \theta ~\leq~ \norm{\BF Z} ~\leq~ \beta \right\}}
	 			~\big|~ \BF H_T  \right]
	&\geq&  \frac{T}{\beta} \,\indicator_{\left\{  \Xi  ~<~ \sqrt{T} \right\}} \E \left[   \Gamma  \, \, 
					\indicator_{\left\{ \theta  ~\leq~ \norm{\BF Z} ~\leq~ \beta \right\}}
								\indicator_{\left\{ \Gamma ~\geq~  \eta \, \Theta \right\}}
	 			~\big|~ \BF H_T \right]\\
	&\geq&  \frac{\eta \, T }{\beta} \, \Theta \,\indicator_{\left\{  \Xi  ~<~ \sqrt{T} \right\}} \E \left[  
					\indicator_{\left\{ \theta  ~\leq~ \norm{\BF Z} ~\leq~ \beta \right\}}
								\indicator_{\left\{  \Gamma ~\geq~  \eta \, \Theta \right\}}
	 			~\Big|~ \BF H_T \right]\\
	&\geq& \frac{ \eta \,   \sqrt{T}  }{2 \beta }  \indicator_{\left\{ \Xi  ~<~ \sqrt{T} \right\}} 	
	\E \left[   \indicator_{\left\{ \theta ~\leq~ \norm{\BF Z} ~\leq~ \beta  \right\}}
								\indicator_{\left\{  \Gamma ~\geq~  \eta \, \Theta \right\}}
	 			~\Big|~ \BF H_T \right]~,
\eex
where the last inequality follows from Lemma \ref{lemma:estimation-and-exploration} which
implies that, with probability one, 
$$
 \frac{\eta \, T }{\beta} \, \Theta \,\indicator_{\left\{  \Xi  ~<~ \sqrt{T} \right\}}
\geq
\frac{ \eta \, T  }{\beta } \cdot \frac{ 1}{r + \Xi}  \indicator_{\left\{ \Xi  ~<~ \sqrt{T} \right\}} 
\geq
\frac{ \eta \, T  }{\beta} \cdot \frac{ 1}{r + \sqrt{T}}  \indicator_{\left\{ \Xi  ~<~ \sqrt{T} \right\}}
\geq
\frac{ \eta \,   \sqrt{T}  }{2 \beta }   \indicator_{\left\{ \Xi  ~<~ \sqrt{T} \right\}}~,
$$
and where the last inequality follows from the fact that $T \geq r^2$, and thus, $1 / \left( r + \sqrt{T} \right) \geq 1/ (2 \sqrt{T})$.

Putting everything together, we obtain
\bex
\E \left[  \norm{\BF Z}  \Xi  + 
					\frac{T \,   \Gamma }{\norm{\BF Z}}
	 			~\Big|~ \BF H_T \right]
	&\geq&  \theta \sqrt{T} \,  \indicator_{\left\{ \Xi  ~\geq~ \sqrt{T } \right\}} \E \left[ 					\indicator_{\left\{ \theta  ~\leq~ \norm{\BF Z} ~\leq~ \beta\right\}}
	 			~\big|~ \BF H_T \right]\\
	&& ~+~  \frac{ \eta \,   \sqrt{T}  }{2 \beta }   \indicator_{\left\{ \Xi  ~<~ \sqrt{T} \right\}} 	
	\E \left[   \indicator_{\left\{ \theta  ~\leq~ \norm{\BF Z} ~\leq~ \beta \right\}}
								\indicator_{\left\{  \Gamma ~\geq~  \eta \, \Theta \right\}}
	 			~\Big|~ \BF H_T \right]~,\\
		&\geq& \min \left\{ \theta, \frac{\eta}{2\beta} \right\}  \sqrt{T}
				\, \E \left[  \indicator_{\left\{ \theta  ~\leq~ \norm{\BF Z} ~\leq~ \beta  \right\}}
								\indicator_{\left\{  \Gamma ~\geq~  \eta \, \Theta \right\}}
	 			~\Big|~ \BF H_T \right]~,
\eex
with probability one. By the  Bonferroni Inequality, we have that
\bex
 \E \left[  
					\indicator_{\left\{ \theta  ~\leq~ \norm{\BF Z} ~\leq~ \beta  \right\}}
								\indicator_{\left\{  \Gamma ~\geq~  \eta \, \Theta \right\}}
	 			~\big|~ \BF H_T \right]
	&=& \Pr \left\{  \theta  ~\leq~ \norm{\BF Z} ~\leq~ \beta  
				\quad \textrm{and} \quad  \Gamma ~\geq~  \eta \, \Theta
	 			~\big|~ \BF H_T \right\}\\
	&\geq& \Pr \left\{  \theta  ~\leq~ \norm{\BF Z} ~\leq~ \beta  
	 			~\big|~ \BF H_T \right\}
			+ \Pr \left\{   \Gamma ~\geq~  \eta \, \Theta
	 			~\big|~ \BF H_T \right\} -1 ~,
\eex
with probability one. 
Conditioned on $\BF H_T$,  
$\left( \BF Z - \widehat{\BF Z}_T\right)^\prime {\BF S}^k_T$
is normally distributed with mean zero
and variance 
$$
\E \left[ \left\{ \left( \BF Z - \widehat{\BF Z}_T\right)^\prime {\BF S}^k_T \right\}^2 ~\Big|~ \BF H_T \right] = \E \left[ \Gamma ~\Big|~ \BF H_T \right] = \Theta~.
$$ 
Let $\Phi(\cdot)$ be the cumulative distribution function of the standard normal random variable, that is,
$\Phi(x) = \frac{1}{\sqrt{2\pi}} \int_{-\infty}^{x} e^{-u^2/2} \, du$. Then,  
\bex
\Pr \left\{   \Gamma ~\geq~  \eta \, \Theta 
	 			~\big|~ \BF H_T \right\}
	&=& 	\Pr \left\{ \big| \left( \BF Z - \widehat{\BF Z}_T\right)^\prime {\BF S}^k_T \big|  \geq  \sqrt{\eta} \sqrt{\Theta}  ~\big|~ \BF H_T \right\}
	= 2 \left( 1- \Phi \left( \sqrt{\eta} \right)  \right)~,
\eex
from which it follows that, with probability one,
$$
 \E \left[  
					\indicator_{\left\{ \theta  ~\leq~ \norm{\BF Z} ~\leq~ \beta  \right\}}
								\indicator_{\left\{  \Gamma ~\geq~  \eta \, \Theta \right\}}
	 			~\Big|~ \BF H_T \right]
	\geq \Pr \left\{  \theta  ~\leq~ \norm{\BF Z} ~\leq~ \beta  
	 			~\big|~ \BF H_T \right\}
			+ 2 \left( 1- \Phi \left( \sqrt{\eta} \right)  \right) -1 ~.
$$
Therefore,
$$
\E \left[  \norm{\BF Z}  \Xi  + 
					\frac{T \,   \Gamma }{\norm{\BF Z}}
	 			~\Big|~ \BF H_T \right]
		\geq \min \left\{ \theta, \frac{\eta}{2 \beta} \right\}  
		\left[
		\Pr \left\{  \theta  ~\leq~ \norm{\BF Z} ~\leq~ \beta  
	 			~\Big|~ \BF H_T \right\}
			+ 2 \left( 1- \Phi \left( \sqrt{\eta} \right)  \right) -1
		\right] \sqrt{T}~,
$$
with probability one, which implies that
\bex
\E \left[  \norm{\BF Z}  \Xi  + 
					\frac{T \,   \Gamma }{\norm{\BF Z}} \right]
		&\geq& \min \left\{ \theta, \frac{\eta}{2 \beta} \right\} 
		\left[
		\Pr \left\{  \theta  ~\leq~ \norm{\BF Z} ~\leq~ \beta   \right\}
			+ 2 \left( 1- \Phi \left( \sqrt{\eta} \right)  \right) -1
		\right]  \sqrt{T}~,\\
		&\geq&  \min \left\{ \theta, \frac{\eta}{2\beta} \right\}   \; 
		\left[
					2 \left( 1- \Phi \left( \sqrt{\eta} \right)  \right) - \frac{1}{\beta^2} - 4 \theta^2
		\right] \, \sqrt{T}~,
\eex
where the last inequality follows from Lemma \ref{lemma:gaussian-lower-bound}.
Set $\theta = 0.09$, $\beta = 3$, and $\eta = 0.5$, to obtain
$
\E \left[  \norm{\BF Z}  \Xi  + 
					\frac{T \,   \Gamma }{\norm{\BF Z}} \right]
	\geq 0.027  \sqrt{T}~,
$
which is the desired result.
\end{proof}

Finally, here is the proof of Theorem \ref{thm:lower-bound}.

\begin{proof}   
It follows from Lemmas \ref{lemma:regret-exploration-estimation} and \ref{lemma:sum-of-exploration-and-estimation} that 
\bex
\Risk \left( T, \BF \psi \right)
	&\geq& \frac{1}{2} \,  \sum_{k=1}^{r-1}    \E \left[  \norm{\BF Z}   \sum_{t=1}^T  \left(\BF U_t^\prime    {\BF S}_T^k  \right)^2 	 		+  \frac{T}{\norm{\BF Z}}   
					\left\{ \left( \BF Z - \widehat{\BF Z}_T \right)^\prime {\BF S}_T^k \right\}^2
	 			 \right]\\
	&\geq& \frac{r-1}{2} \cdot 0.027 \sqrt{T} ~\geq~ \frac{r}{4} \cdot 0.027 \sqrt{T} ~\geq~  0.006 \,r\, \sqrt{T}~,
\eex
where we have used the fact $r\geq 2$, which implies that $r-1 \geq r/2$. 
\end{proof}

\section{Matching Upper Bounds}
\label{section:matching-upper-bound}

We have established $\Omega\left( r \sqrt{T} \right)$ lower bounds when the set of arms $\CAL U_r$ is the unit sphere.  We now prove that a policy that alternates between exploration and exploitation phases yields matching upper bounds on the regret and risk, and is therefore optimal for this problem.    Surprisingly, we will see that the phase-based policy is effective for a large class of bandit problems, involving a strongly convex set of arms. 
We introduce the following  assumption on the tails of the error random variables $W_t^{\BF u}$ and on the set of arms $\CAL U_r$, which will remain in effect throughout  the rest of paper.

\begin{assumption}\label{assumption:basic} 
~\begin{enumerate}
\item[(a)]  There exists a positive constant $\sigma_0$ 
such that for any $r \geq 2$, $\BF u \in \CAL U_r$, $t \geq 1$, and $x \in \Real$, we have $\E \left[ e^{x W_t^\BF u} \right] \leq e^{x^2 \sigma_0^2 / 2}$ .
\item[(b)]  There exist positive constants $\bar{u}$ and $\lambda_0$ such that for any $r \geq 2$,
$$
	\max_{\BF u \,\in\, \CAL U_r} \norm{\BF u} \leq \bar{u}~,
$$ 
and the set of arms $\CAL U_r \subset \Real^r$ has $r$ linearly independent elements $\BF b_{1}, \ldots, \BF b_{r}$ such that $\lambda_{\min} \left( \sum_{k=1}^r \BF b_{k} \BF b_{k}^\prime \right) \geq \lambda_0$.
\end{enumerate}
\end{assumption}
\noindent Under Assumption \ref{assumption:basic}(a), 
the tails of the distribution of the errors $W^{\BF u}_t$ decay at least as fast as  for a normal
random variable with variance $\sigma_0^2$. The first part of 
Assumption \ref{assumption:basic}(b) ensures that the expected reward of the arms remain bounded as the dimension $r$ increases, while the arms $\BF b_1, \ldots , \BF b_r$ given in the second
part of Assumption \ref{assumption:basic}(b) will be used during the exploration phase of our 
policy.

Our policy -- which we refer to as the \textsc{Phased Exploration and Greedy Exploitation (PEGE)}  --   operates in cycles, and in each cycle, we alternate between exploration and exploitation phases.  During
the exploration phase of cycle $c$, we play the $r$ linearly independent arms
from Assumption  \ref{assumption:basic}(b).  Using the rewards observed during the exploration phases in the past $c$ cycles, we compute
an {\em ordinary least squares} (OLS) estimate $\widehat{\BF Z}(c)$. 
In the exploitation phase of cycle $c$, we  use $\widehat{\BF Z}(c)$ as a proxy for $\BF Z$ and compute a {\em greedy} decision $\BF G(c) \in \CAL U_r$ defined by:
\beq
			\BF G(c) = \arg \max_{ \BF v \, \in \, \CAL U_r} \BF v^\prime \widehat{\BF Z}(c)~, \label{eq:greedy-decision}
\eeq
where we break ties arbitrarily. 
We then play the arm $\BF G(c)$ for an additional $c$ periods to complete cycle $c$.    
Here is a formal description of the policy.

\vspace{0.1in}
\noindent \underline{\textsc{Phased Exploration and Greedy Exploitation (PEGE)}}


\noindent {\bf Description:} For each cycle $c \geq 1$, complete the following two phases.

\begin{enumerate}
\item {\bf Exploration ($r$ periods):}  For $k = 1, 2, \ldots, r$, play  arm $\BF b_k \in \CAL U_r$ given in Assumption~\ref{assumption:basic}(b), and observe the reward $X^{\BF b_k}(c)$. 
Compute the OLS estimate $\widehat{\BF Z}(c) \in \Real^r$, given by 
$$
\widehat{\BF Z}(c) =  \frac{1}{c} \left( \sum_{k=1}^r \BF b_k \BF b_k^\prime \right)^{-1} \sum_{s=1}^c \sum_{k=1}^r \BF b_k X^{\BF b_k} (s)  = \BF Z + \frac{1}{c} \left( \sum_{k=1}^r \BF b_k \BF b_k^\prime \right)^{-1} \sum_{s=1}^c \sum_{k=1}^r \BF b_k W^{\BF b_k} (s)~,
$$
where for any $k$, $X^{\BF b_k} (s)$ and $W^{\BF b_k}(s)$ denote the observed reward and the error
random variable associated with playing arm $\BF b_k$ in cycle $s$.  Note that the
last equality follows from Equation (\ref{eq:model}) defining our model.
\item {\bf Exploitation ($c$ periods):}  
Play the greedy arm  $\BF G(c) = \arg \max_{\BF v \in \CAL U_r} \BF v^\prime \widehat{\BF Z}(c)$
for $c$ periods.
\end{enumerate}
\vspace{0.2in}

Since $\CAL U_r$ is compact, for each $\BF z \in \Real^r$, there is an
optimal arm 
that gives the maximum expected
reward.
When this best arm varies smoothly with $\BF z$,
we
will  show that the $T$-period regret and risk under the \textsc{PEGE} policy is bounded above by $O( r \sqrt{T} )$. 
More precisely, we say that
a set of arms~$\CAL U_r$ satisfies the
{\em smooth best arm response with parameter $J$} (SBAR($J$), for short) condition if for any
nonzero vector $\BF z \in \Real^r \setminus \{\BF 0\}$, there is a unique best arm $\BF u^*(\BF z) \in \CAL U_r$ that
gives the maximum expected reward, and 
for any two unit vectors $\BF z \in \Real^r$ an $\BF y \in \Real^r$
 with $\norm{\BF z} = \norm{\BF y} = 1$, we have
$$
	 \norm{ \BF u^*(\BF z)  - \BF u^*(\BF y) } \leq J \norm{ \BF z  - \BF y}~.
$$ 

Even though the SBAR condition appears to be an implicit one, it admits a simple interpretation. According to Corollary 4 of \citet{Polovinkin:1996}, a compact set $\CAL U_r$ satisfies condition SBAR($J$) if and only if it is strongly convex with parameter $J$, in the sense that the set $\CAL U_r$ can be represented as the intersection of closed balls of radius $J$. Intuitively,
the SBAR condition requires the boundary of $\CAL U_r$ to have a curvature that is bounded below by a positive constant. For some examples, the unit ball satisfies the SBAR(1) condition.
Furthermore, according to Theorem 3 of \cite{Polovinkin:1996}, an ellipsoid of the form $\{\BF u \in R^r : \BF u'\BF Q^{-1} \BF u \leq 1\}$,
where $\BF Q$ is a symmetric positive definite matrix, satisfies the condition 
SBAR $\left(\lambda_{\max}(\BF Q)/\sqrt{\lambda_{\min}(\BF Q)}\right)$.

The main result of this section is stated in the following theorem.  The proof is given in 
Section~\ref{section:greedy-regret-analysis}.

\begin{thm}[Regret and Risk Under the Greedy Policy]  \label{theorem:greedy-regret} 
Suppose that
Assumption \ref{assumption:basic} holds and that
the sets $\CAL U_r$ satisfy the SBAR($J$) condition.
Then, there exists a positive constant $a_1$  that depends only on $\sigma_0$, $\bar{u}$,  $\lambda_0$, and $J$,  such that  for any $\BF z \in \Real^r \setminus \{\BF 0\}$ and $T \geq r$, 
$$	
{\rm Regret} \left( \BF z, \, T,  \, \textsc{PEGE} \right)
	\leq a_1 \, \left( \norm{\BF z} + \frac{1}{\norm{\BF z}} \right) \,  r    \sqrt{T} ~.
$$
Suppose in addition, that
there exists a constant $M > 0$ 
such that for every $r\geq 2$ we have
$\E \left[\, \norm{\BF Z} \,\right] \leq M$
and
$\E \left[ \,1/\norm{\BF Z} \,\right] \leq M$.
Then, there exists a positive constant $a_2$ that depends only on 
$\sigma_0$, $\bar{u}$,  $\lambda_0$, $J$, and $M$,  such that 
for any $T \geq r$,
$$	
{\rm Risk} \left( T,  \, \textsc{PEGE} \right)
	\leq a_2 \, r  \,  \sqrt{T} ~.
$$
\end{thm}

\pr{{\bf Dependence on $\norm{\BF z}$ in the regret bound:}
By Assumption \ref{assumption:basic}(b),  for any $\BF z \in \Real^r$,
the instantaneous regret under arm $\BF v \in \CAL U$ is \jnt{bounded} by $\max_{\BF u \in \CAL U} \BF z^\prime(\BF u - \BF v) \leq 2 \bar{u} \norm{\BF z}$.  Thus, 
$2 \bar{u} \norm{\BF z} T$ provides a trivial upper bound on the $T$-period cumulative
regret under the \textrm{PEGE} policy.  Combining
this  with Theorem \ref{theorem:greedy-regret}, we have that
$$
\Regret( \BF z, T, \textrm{PEGE} ) \leq \max\{a_1, 2 \bar{u}\} \cdot \min \left\{ 
\left( \norm{\BF z} + \frac{1}{\norm{\BF z}} \right) \,  r    \sqrt{T},~   \norm{\BF z} T \right\}~.
$$
The above result shows that the performance of our policy does {\em not} deteriorate as
the norm of $\BF z$ approaches zero.}

Intuitively, the requirement $\E \left[ \norm{\BF Z} \right] \leq M$
in 
Theorem \ref{theorem:greedy-regret} implies that, as $r$ increases,
the maximum expected reward (over all arms) 
remains bounded.
Moreover, the assumption on the 
boundedness of $\E \left[ \,1/\norm{\BF Z} \,\right]$ 
means 
that $\BF Z$ does not have too much mass near the origin.
The following lemma provides conditions under which this assumption holds, and shows that the case of 
the multivariate
normal distribution used in Theorem \ref{thm:lower-bound} is also covered. The proof is
given in Appendix \ref{section:proof-gaussian-has-small-mass}.

\begin{lem}[Small Mass Near the Origin]  \label{lemma:gaussian-has-small-mass} 
~\begin{enumerate}
\item[(a)] Suppose that there exist constants $M_0$ and $\rho\in(0,1]$ such that 
for any $r\geq 2$, the random variable $\norm{ \BF Z}$ has a density function 
$g : \Real_+  \to \Real_+$ such that $g(x)\leq M_0 x^{\rho}$ for all $x\in[0,\rho]$. Then, $\E \left[ \,1/\norm{\BF Z} \,\right] \leq M$, where
$M$ depends only on $M_0$ and $\rho$.
\item[(b)] Suppose that for any $r \geq 2$, the random vector  $\BF Z$ has a multivariate normal distribution with
mean $\BF 0 \in \Real^r$ and covariance matrix $\BF I_r / r$. Then,
$\E \left[ \,\norm{\BF Z} \,\right]\leq 1$ and
$\E \left[ \,1/\norm{\BF Z} \,\right] \leq \sqrt{\pi}$.
\end{enumerate}
\end{lem}

The following corollary shows that the example in  Section \ref{section:lower-bound} admits
tight matching upper bounds on the regret and risk.

\begin{corollary}[Matching Upper Bounds] \label{co:matching-upper-bounds}
Consider a bandit problem where the set of arms is the unit sphere in $\Real^r$, and where $W^{\BF u}_t$ has a standard normal distribution with mean zero and variance one for all $t$ and $\BF u$. 
Then,  there exists an absolute constant $a_3$ such that 
for any $\BF z \in \Real^r \setminus \{\BF 0\}$ and $T \geq r$,
$$
	{\rm Regret} \left( \BF z, T, \textsc{PEGE} \right) \leq a_3 \, \left( \norm{\BF z} + \frac{1}{\norm{\BF z}} \right) r \sqrt{T}~.
$$
Moreover, if $\BF Z$ has a multivariate normal distribution with mean $\BF 0$ and covariance matrix $\BF I_r / r$, then for all $T \geq r$,
$$
	{\rm Risk} \left( T, \textsc{PEGE} \right) \leq a_3 \, r \sqrt{T}~.
$$
\end{corollary}
\begin{proof} Since the set of arms is the unit sphere and the errors are standard normal, 
Assumption \ref{assumption:basic} is satisfied with $\sigma_0 = \bar{u} = \lambda_0 = 1$.
Moreover, as already discussed,  the unit sphere 
satisfies the SBAR(1) condition. Finally,
By Lemma \ref{lemma:gaussian-has-small-mass}, the random
vector $\BF Z$ satisfies the hypotheses of Theorem \ref{theorem:greedy-regret}. 
The regret and risk bounds then follow immediately.
\end{proof}

\subsection{Proof of Theorem \ref{theorem:greedy-regret}}
\label{section:greedy-regret-analysis}

The proof of Theorem \ref{theorem:greedy-regret} relies on the following upper bound on the square of the norm difference 
between $\widehat{\BF Z}(c)$ and   $\BF Z$. 

\begin{lem}[Bound on Squared Norm Difference] \label{lemma:square-difference} Under Assumption \ref{assumption:basic}, there exists a positive constant $h_1$ that depends only on $\sigma_0$, $\bar{u}$, and $\lambda_0$ such that for any $\BF z  \in \Real^r$ and $c \geq 1$,
$$
	\E \left[  \norm{ \widehat{\BF Z}(c) - \BF z }^2 ~\Big|~ \BF Z = \BF z \right] \leq \frac{ h_1 \, r}{c}~.
$$
\end{lem}
\begin{proof}
Recall from the definition of the \textsc{PEGE} policy that
the estimate $\widehat{\BF Z}(c)$ at the end of the exploration phase of cycle $c$ is given by
$$
\widehat{\BF Z}(c)  = \BF Z + \frac{1}{c} \left( \sum_{k=1}^r \BF b_k \BF b_k^\prime \right)^{-1} \sum_{s=1}^c \sum_{k=1}^r \BF b_k W^{\BF b_k} (s) = \BF Z + \frac{1}{c} \sum_{s=1}^c \BF B \, \BF V(s)~,
$$
where $\BF B = \left( \sum_{k=1}^r \BF b_k \BF b_k^\prime \right)^{-1}$ and $\BF V(s) = 
\sum_{k=1}^r \BF b_k W^{\BF b_k} (s)$.  Note that the mean-zero random variables $W^{\BF b_k} (s)$ are 
independent of each other and their variance is bounded by some constant $\gamma_0$ that
depends only on $\sigma_0$.  Then, it follows from Assumption \ref{assumption:basic} that
\bex
	\E \left[  \norm{ \widehat{\BF Z}(c) - \BF z }^2 ~\Big|~ \BF Z = \BF z \right]  &=& \frac{1}{c^2} 
	\sum_{s=1}^c \E \left[ \BF V(s)^\prime \BF B^2 \BF V(s) \right]
	= \frac{1}{c^2} 
	\sum_{s=1}^c \sum_{k=1}^r \E \left[ \left( W^{\BF b_k}(s) \right)^2 \right] \BF b_k^\prime \BF B^2 \BF b_k \\
	&\leq& \frac{\gamma_0}{c} 
	\sum_{k=1}^r  \BF b_k^\prime \BF B^2 \BF b_k 
	\leq \frac{\gamma_0}{c} 
	\sum_{k=1}^r  \lambda_{\max}\left( \BF B^2 \right) \norm{\BF b_k}^2  
	\leq \frac{\gamma_0 \, \bar{u}^2  \, r}{\lambda_0^2  \, c}~, 
\eex
which is the desired result.
\end{proof}

The next lemma gives an upper bound on the difference between two normalized vectors in terms of the difference of the original vectors.  

\begin{lem}[Difference Between Normalized Vectors] \label{lemma:normalized-difference} For any $\BF z,\,\BF w \in \Real^r$, not both equal to zero, 
$$
	\norm{ \frac{ \BF w}{\norm{\BF w}} -  \frac{ \BF z}{\norm{\BF z}} }
	\leq \frac{2 \, \norm{ \BF w - \BF z}}{ \max \left\{\norm{\BF z}, \norm{\BF w} \right\}} ~,
$$
where we define $\BF 0 / \norm{\BF 0}$ to be some fixed unit vector.
\end{lem}
\begin{proof}
The inequality is easily seen to hold if either $\BF w = \BF 0$ or $\BF z = \BF 0$.  So, assume that both $\BF w$ and $\BF z$ are nonzero.   Using the triangle inequality and the fact that $\big| \norm{\BF w} - \norm{\BF z}
 \big| \leq \norm{ \BF w - \BF z}$, we have that
$$
	\norm{ \frac{ \BF w}{\norm{\BF w}} -  \frac{ \BF z}{\norm{\BF z}} }
	\leq 	\norm{ \frac{ \BF w}{\norm{\BF w}} -  \frac{ \BF z}{\norm{\BF w}} }
		+ 	\norm{ \frac{ \BF z}{\norm{\BF w}} -  \frac{ \BF z}{\norm{\BF z}} }
	= \frac{ \norm{  \BF w   - \BF z} }{\norm{\BF w}} + \norm{\BF z}
	\abs{ \frac{1}{\norm{\BF w}} - \frac{1}{\norm{\BF z}}} \leq \frac{ 2 \norm{\BF w - \BF z}}{ \norm{\BF w}}~.
$$	
By symmetry, we also have
$
\norm{ \frac{ \BF w}{\norm{\BF w}} -  \frac{ \BF z}{\norm{\BF z}} }
	\leq \frac{2 \norm{ \BF w  - \BF z }}{\norm{\BF z}}$,
which gives the desired result.
\end{proof}

The following lemma gives an upper bound on the expected instantaneous regret under
the greedy decision $\BF G(c)$ during the exploitation phase of cycle $c$.

\begin{lem}[Regret Under the Greedy Decision] \label{lemma:instantaneous-regret-greedy}
Suppose that Assumption \ref{assumption:basic} holds and the sets $\CAL U_r$ satisfy the SBAR($J$) condition. Then,
there exists a positive constant $h_2$ that depends only on $\sigma_0$, $\bar{u}$, $\lambda_0$, and $J$, such that for any $\BF z \in \Real^r$ and $c \geq 1$,
$$
	\E \left[ \max_{\BF u \in \CAL U^r} \BF z^\prime \left( \BF u - \BF G(c) \right)   ~\Big|~ \BF Z = \BF z \right] \leq \frac{r \, h_2}{c \norm{\BF z}} ~,
$$
\end{lem}
\begin{proof} The result is trivially true when $\BF z = \BF 0$. So, let 
us fix some $\BF z \in \Real^r \setminus \{\BF 0\}$.  
By comparing the greedy decision $\BF G(c)$ with the best arm $\BF u^*(\BF z)$, we see that the 
instantaneous regret satisfies
\bex
	\BF z^\prime \left( \BF u^*\left({\BF z} \right) - \BF G(c)  \right) 
		&=& 	\left( \BF z - \widehat{\BF Z}(c) \right)^\prime \BF u^*\left({\BF z} \right)
			 	+ 	 \left(\BF u^*\left({\BF z} \right)- \BF G(c)  \right)^\prime \widehat{\BF Z}(c) 
				+  \left( \widehat{\BF Z}(c) -  \BF z \right)^\prime \BF G(c)  \\
		&\leq& \left( \BF z - \widehat{\BF Z}(c) \right)^\prime \BF u^*\left({\BF z} \right)
				+  \left( \widehat{\BF Z}(c) -  \BF z \right)^\prime \BF G(c)  \\ 
		&=& \left( \widehat{\BF Z}(c) -  \BF z \right)^\prime \left( \BF G(c)  -  \BF u^*\left({\BF z} \right) \right)
		= \left( \widehat{\BF Z}(c) -  \BF z \right)^\prime \left(  \BF u^*\left( \widehat{\BF Z}(c)  \right)  -  \BF u^*\left({\BF z} \right) \right) ~,
\eex
where the inequality follows from the definition of the greedy decision in Equation (\ref{eq:greedy-decision}), and the final equality follows from the fact that
$\BF G(c) = \BF u^*\left( \widehat{\BF Z}(c) \right)$.  As a convention,
we define $\BF 0 / \norm{\BF 0}$ to some fixed unit vector and 
set $\BF u^*(\BF 0) = \BF u^*(
\BF 0 / \norm{\BF 0})$.
 
It then follows
from the Cauchy-Schwarz Inequality that, with probability one, 
\bex
	\BF z^\prime \left( \BF u^*(\BF z) - \BF G(c) \right) 
		&\leq& \norm{ \widehat{\BF Z}(c) -  \BF z }  \norm{ \BF u^*\left( \widehat{\BF Z}(c)  \right)  -  \BF u^*(\BF z)  } \\
		&=& \norm{ \widehat{\BF Z}(c) -  \BF z }  
		\norm{ \BF u^* \left( \frac{\widehat{\BF Z}(c)}{\|\widehat{\BF Z}(c)\|}  \right)  -  \BF u^*
		\left( \frac{\BF z}{\norm{\BF z}} \right)  } \\  
		&\leq&  J \norm{ \widehat{\BF Z}(c) -  \BF z }  \norm{  \frac{ \widehat{\BF Z}(c) }{ \|\widehat{\BF Z}(c)\|}  -  \frac{\BF z}{\norm{\BF z}} } 
		~\leq~  \frac{ 2 J \norm{ \widehat{\BF Z}(c) -  \BF z }^2 }{ \norm{\BF z} }~,
\eex
where the equality follows from the fact that $\BF u^*(\BF z) = \BF u^*( \lambda \BF z)$ for all
$\lambda > 0$.  The second inequality follows from {condition SBAR($J$),} and the final inequality follows from Lemma \ref{lemma:normalized-difference}.  The desired result follows by taking
conditional expectations, given ${\bf Z=z}$, and applying Lemma~\ref{lemma:square-difference}.
\end{proof}



We can now complete the proof of Theorem \ref{theorem:greedy-regret}, by adding the regret over the differnt times and cycles.
By Assumption \ref{assumption:basic} and the Cauchy-Schwarz Inequality, the instantaneous regret from playing any arm $\BF u \in \CAL U_r$ is bounded above by 
$\max_{\BF v \in \CAL U_r} \BF z^\prime \left( \BF v - \BF u \right) \leq 2 \, \bar{u} \norm{\BF z}$.  Consider an arbitrary cycle~$c$.   Then, the total regret incurred during the exploration phase (with $r$ periods) in this cycle is bounded above by $2  \, \bar{u} \, r \, \norm{\BF z}$. 
During the exploitation phase of cycle $c$, we always play the greedy arm  $\BF G(c)$.
The expected instantaneous regret in each period during the exploitation phase is bounded above by $rh_2/c \norm{\BF z}$.
So, the total regret during cycle $c$ is bounded above by
$2  \, \bar{u} \, r \, \norm{\BF z} +  
h_2 \, r/ \norm{\BF z}  $. 
Summing over $K$ cycles, we obtain
$$	
{\rm Regret} \left( \BF z, \, r K + \sum_{c=1}^K c,  \, \textsc{PEGE} \right)
	\leq  h_3 \, r   \norm{\BF z} K  +  h_4 \sum_{c=1}^K \frac{r}  {\norm{\BF z}} ~,
$$
for some positive constants $h_3$ and $h_4$ that depend
only on $\sigma_0$, $\bar{u}$, $\lambda_0$, and $J$.

 Consider an arbitrary time period $T \geq r$ and $\BF z \in \Real^r$.  Let $K_0 = \left \lceil  \sqrt{ 2 T} \right \rceil$.  Note that the total time periods after $K_0$ cycles is at least $T$ because
$
	r K_0 + \sum_{c=1}^{K_0} c  \geq \sum_{c=1}^{K_0} c = \frac{K_0(K_0+1)}{2} \geq \frac{K_0^2}{2} \geq T~.
$
Since the cumulative regret is nondecreasing over time, it follows that
\bex
		\Regret\left( \BF z, T, \textsc{PEGE} \right) &\leq& {\rm Regret} \left( \BF z, \, r K_0 + \sum_{c=1}^{K_0} c,  \, \textsc{PEGE} \right)\\
		&\leq& h_3 \, r \norm{\BF z} K_0 + h_4 \frac{r K_0}{  \norm{\BF z}} 
		~\leq~  3 \max\{h_3, h_4\} \left( \norm{\BF z}  + \frac{1}{\norm{\BF z}} \right) r \sqrt{T}~,
\eex
where the final inequality follows because  $K_0 = \left \lceil  \sqrt{ 2 T} \right \rceil \leq 3 \sqrt{T}$.
The risk bound follows by taking expectations and using the assumption on the boundedness of $\E[ \,\norm{\BF Z}\,]$ and $\E[\, 1/ \norm{\BF Z}\,]$.


\section{A Policy for General Bandits}
\label{section:active-exploration}

We have shown that when a bandit has a smooth best arm response, the
\textsc{PEGE} policy 
achieves optimal $O( r \sqrt{T} )$ regret and Bayes risk. 
The general idea is that  when the estimation error is small, the 
instantaneous regret of the greedy
decision based on our estimate $\widehat{\BF Z}(c)$ can be of the same order as
$\|\BF Z - \widehat{\BF Z}(c)\| $.  However, under
the smoothness assumption, this upper bound on the instantaneous regret is
improved to $O \left(\|\BF Z - \widehat{\BF Z}(c)\|^2\right)$, as shown in the proof of Lemma \ref{lemma:instantaneous-regret-greedy}, and this enables us to separate exploration from exploitation. 

However, if the number of arms is finite or if the collection of arms is an arbitrary compact set,  then the \textsc{PEGE} policy
may not be effective. This is because a small estimation error may have a disproportionately large effect on the arm chosen by a greedy policy, leading to a large instantaneous regret.
In this section, we discuss
a policy -- which we refer to as the \textsc{Uncertainty Ellipsoid (UE)} policy -- that can be applied to any bandit problem, at the price of slightly higher regret and Bayes risk.  In contrast to the \textsc{PEGE} policy, the \textsc{UE} policy  combines  active exploration and exploitation in every  period.  

As discussed in the introduction, the \textsc{UE} policy is 
closely related to 
the algorithms described 
in \citet{Auer:2002} and \citet{DHK:2008a}, but also has the ``anytime'' property (the policy does not require  prior knowledge of the time horizon $T$), and also allows the random vector $\BF Z$ and the errors 
$W^{\BF u}_t$ to be unbounded.  
For the sake of completeness, we give
a detailed description of our policy and state the regret and risk bounds that we obtain.  The reader can find the proofs of these bounds in Appendix \ref{appendix:UE-proofs}.


To facilitate exposition, we introduce a constant that 
will appear 
in the description of the policy, namely,
\be
	   \kappa_0 = 2 \sqrt{ 1 +  \log \left(1 + \frac{ 36 \, \bar{u}^2 }{ \lambda_0 }\right)}~,
	   \label{eq:constant-bounds}
\ee
where the parameters $\bar{u}$ and $\lambda_0$ are given in Assumption \ref{assumption:basic}.
The \textsc{UE} policy maintains, at each time period $t$,
the following two pieces of information.
\begin{enumerate}
\item The {\em ordinary least squares (OLS) estimate} defined as follows:
if $\BF U_1, \ldots, \BF U_t$ are the arms chosen  during the first $t$ periods, then the OLS estimate
$\widehat{ \BF Z}_t$ is given by\footnote{Let us note that we are abusing notation here.  Throughout
this section $\widehat{\BF Z}_t$ stands for the OLS estimate, which is different from the least mean squares 
estimator $\E \left[ \BF Z ~\big|~ \BF H_t \right]$ introduced in Section \ref{section:lower-bound}.}:
\be
	\BF C_t				= \left( \sum_{s=1}^t \BF U_s \BF U_s^\prime \right)^{-1}~,
	    \quad
	    	\BF M_t = \sum_{s=1}^t \BF U_s W_s~~,
	\quad \textrm{ and } \quad
	\widehat{\BF Z}_t =   \BF C_t \sum_{s=1}^t   \BF U_s X_s
	= \BF Z + \BF C_t \BF M_t~. 
	\label{eq:OLS}
\ee
In contrast to the \textsc{PEGE} policy, whose estimates relied only on 
the  rewards observed in the exploration phases, 
the estimate 
$\widehat{\BF Z}_t$ incorporates {\em all} available information up to time $t$. 
We initialize the policy by playing $r$ linearly independent arms, so that
$\BF C_t$ is positive definite for $t\geq r$.

%

\item An {\em uncertainty ellipsoid } ${\CAL E}_t \subseteq \Real^r$  associated with the estimate $\widehat{ \BF  Z}_t$, defined by,
\be
	{\CAL E}_t  = \left\{ \BF w \in \Real^r :  \BF w^\prime  \BF C_t^{-1} \BF w  \leq  
	 \left( \alpha \, \sqrt{\log t} \,
		\sqrt{ \min\{ r \log t \,,\, \abs{\CAL U_r} \}}  \right)^2  \right\} 
		\quad \textrm{and} \quad
	\alpha = 4 \, \sigma_0 \, \kappa_0^2    ~,
	 \label{eq:uncertainty-ellipsoid}
\ee
where the parameters $\sigma_0$ and $\kappa_0$ are given 
in Assumption \ref{assumption:basic}(a) and Equation (\ref{eq:constant-bounds}).
The uncertainty ellipsoid ${\CAL E}_t$ represents the set of likely ``errors'' associated with the estimate 
$\widehat{\BF Z}_t$. 
We define the {\em uncertainty radius} $R_t^\BF u$
associated with each arm $\BF u$ as follows:
\be
	R_t^\BF u = \max_{\BF v \in {\CAL E}_t } \BF v^\prime \BF u 
		= \alpha \, \sqrt{\log t} \,
		\sqrt{ \min\{ r \log t \,,\, \abs{\CAL U_r} \}}\left\| \BF u \right\|_{\BF C_t} ~. \label{eq:uncertainty-radius}
\ee
\end{enumerate}

A formal description of the policy
is given below. 

\vspace{0.2in}
\noindent \underline{\textsc{Uncertainty Ellipsoid (UE)}}

\noindent {\bf Initialization:} During the first $r$ periods, play the $r$ linearly independent arms 
$\BF b_{1}, \BF b_{2}, \ldots, \BF b_{r}$ given in Assumption \ref{assumption:basic}(b).
Determine the OLS estimate $\widehat{\BF Z}_r$, the uncertainty ellipsoid ${\CAL E}_r$,
and
the uncertainty radius associated with each arm.

\noindent {\bf Description:}  For $t \geq r+1$, do the following:
\begin{itemize}
\item[(i)] Let $\BF U_t \in \CAL U_r$ be an arm that gives the maximum estimated reward over the ellipsoid $\widehat{\BF Z}_{t-1} + {\CAL E}_{t-1}$, that is,
\be
		\BF U_t = \arg \max_{\BF v \in \CAL U_r} 
				  		\left\{  \BF v^\prime \widehat{\BF Z}_{t-1}+ 
						\max_{\BF w \in {\CAL E}_{t-1}}
 						\BF w^\prime \BF v \right\}
					= \arg \max_{\BF v \in \CAL U_r} \left\{  \BF v^\prime \widehat{\BF Z}_{t-1} + R_{t-1}^\BF v
						\right\}~, \label{eq:UE-decision}
\ee
where the uncertainty radius $R_{t-1}^{\BF v}$ is defined in Equation (\ref{eq:uncertainty-radius}); ties are broken arbitrarily.  
\item[(ii)] Play arm $\BF U_t$ and observe the resulting reward $X_t$.
\item[(iii)] Update 
the OLS estimate $\widehat{\BF Z}_t$, the uncertainty ellipsoid
${\CAL E}_t$, and the uncertainty radius $R_t^\BF u$ of each arm $\BF u$, using the formulas in Equations (\ref{eq:OLS}), (\ref{eq:uncertainty-ellipsoid}), and (\ref{eq:uncertainty-radius}). 
\end{itemize}
\vspace{0.05in}

By choosing an arm that maximizes the estimated reward over the 
ellipsoid $\widehat{\BF Z}_{t} + {\CAL E}_{t}$, our policy
involves simultaneous 
exploitation (via the term $\BF v^\prime \widehat{\BF Z}_{t}$) and exploration (via the term $R_{t}^\BF v = 
\max_{\BF w \in {\CAL E}_{t}} \BF w^\prime \BF v$) in every period. 
The ellipsoid ${\CAL E}_t$ reflects the uncertainty in our OLS estimate $\widehat{\BF Z}_t$.
It generalizes the classical upper confidence index introduced by \cite{LaiRobbins:1985}, to account for correlations among the arm rewards.  
In the special case of $r$  independent arms where $\CAL U_r = \left\{ \BF e_1, \ldots, \BF e_r \right\}$, it is easy to verify that for each arm $\BF e_\ell$,
the expression $\BF e_\ell^\prime \widehat{\BF Z}_t +  R_t^{\BF e_\ell}$ coincides (up to a scaling constant)
with the upper confidence bound used by \cite{AuerCF:2002}.   
Our definition of the uncertainty radius involves an extra factor of $\sqrt{ \min\{ r\log t, \abs{\CAL U_r}\}}$,
 in order to handle
the case where 
the arms are not standard unit vectors, and the rewards are correlated.

The main results of this section are given in the following two theorems. 
The first theorem establishes upper bounds on the regret and risk when
the set of arms is an arbitrary compact set.
This result shows that the \textsc{UE} policy
is nearly optimal, admitting
upper bounds that are within a logarithmic factor of the $\Omega( r \sqrt{T} )$ 
lower bounds given in Theorem \ref{thm:lower-bound}.
Although the proof of this theorem makes use of somewhat different (and novel) 
large deviation inequalities for 
adaptive 
least squares estimators,
the argument
shares similarities
with the proofs given in \citet{DHK:2008a}, and we omit the details. 
The reader can find a complete proof in Appendix \ref{appendix:proof-UE-regret-analysis}. 

\begin{thm}[Bounds for General Compact Sets of Arms] \label{theorem:UE-regret-analysis}  Under Assumption \ref{assumption:basic}, there exist positive
constants $a_4$ and $a_5$ that depend only on the parameters $\sigma_0$, $\bar{u}$, 
and $\lambda_0$,
such that for all $T \geq r+1$ and $\BF z \in \Real^r$,
$$
{\rm Regret} \left ( \BF z, T, \textsc{UE} \right) ~\leq~
	a_4 r \norm{\BF z}  + a_5 \, r \, \sqrt{T} \, \log^{3/2} T~,
$$
and
$$
	{\rm Risk} \left( T, \textsc{UE} \right) \leq a_4 r \, \E \left[ \norm{\BF Z} \right]  + a_5 \, r \, \sqrt{T} \, \log^{3/2} T~.
$$
\end{thm}

 For  any arm $\BF u \in \CAL U_r$ and $\BF z \in \Real^r$, 
let $\Delta^{\BF u}\left( \BF z \right)$ denote the difference between the maximum expected reward and the expected reward of arm $\BF u$ when $\BF Z = \BF z$, that is,
$$
\Delta^{\BF u} \left( \BF z \right) =  \max_{\BF v \in \CAL U_r} \BF v^\prime \BF z - \BF u^\prime \BF z~.
$$ 
When the number of arms is finite,
it turns out that we can obtain bounds on regret
and risk that scale more gracefully over time, 
growing as $\log T$ and $\log^2 T$, respectively.
This result is stated in Theorem \ref{thm:regret-and-risk-small-arms},
which shows that, for a fixed set of arms,
the \textsc{UE} policy 
is asymptotically  optimal as a function time,
within a constant factor of the lower
bounds established by \citet{LaiRobbins:1985} and \citet{Lai:1987}.

\begin{thm}[Bounds for Finitely Many Arms] \label{thm:regret-and-risk-small-arms}
 Under Assumption \ref{assumption:basic}, there exist positive
constants $a_6$ and $a_7$ that depend only on the parameters $\sigma_0$, $\bar{u}$, 
and $\lambda_0$ 
such that for all $T \geq r+1$ and $\BF z \in \Real^r$,
$$
{\rm Regret} \left ( \BF z, T, \textsc{UE} \right) ~\leq~
a_6 \, \abs{\CAL U_r} \norm{\BF z}   + a_7 \, 
	  \abs{\CAL U_r} \,  \sum_{\BF u \in \CAL U_r} \min \left\{ \frac{ \log T}{\Delta^{\BF u} \left( \BF z \right)} ~,~ T \Delta^{\BF u}(\BF z) \right\}~.
$$
Moreover, suppose that
there exists a positive constant $M_0$ such that, for all arms $\BF u$, the distribution of the random variable 
$\Delta^{\BF u} \left( \BF Z \right)$ is described by a point mass at $0$, and a density function that is 
bounded above by $M_0$ on $\Real_+$. Then, there exist positive
constants $a_8$ and $a_9$ that depend only on the parameters $\sigma_0$,  $\bar{u}$, 
$\lambda_0$, and $M_0$,
such that for all $T \geq r+1$,
$$
	{\rm Risk} \left( T, \textsc{UE} \right) \leq
	a_8 \, \abs{\CAL U_r}\, \E \left[ \norm{\BF Z} \right]  + a_9 \,  \abs{\CAL U_r}^2 \, \log^2 T~.
$$
\end{thm}
\begin{proof}
For any arm $\BF u \in \CAL U_r$ and $\BF z \in \Real^r$, let the random variable $N^\BF u(\BF z, T)$ denote the total number of
times that the arm $\BF u$ is chosen during periods $1$ through $T$, given that $\BF Z = \BF z$.
Using an argument similar to the one in \cite{AuerCF:2002}, we can show
that
$$
	\E \left[ N^\BF u(\BF z, T) \mid \BF Z = \BF z\right] \leq 6 
		+\frac{4 \alpha^2  \abs{\CAL U_r}  \log T }{ \left( \Delta^{\BF u} \left( \BF z \right) \right)^2} ~.
$$
The reader can find a  proof of this result in Appendix \ref{appendix:proof-regret-and-risk-small-arms}.

The regret bound in Theorem \ref{thm:regret-and-risk-small-arms} then follows immediately from the above upper bound
and the fact that $N^\BF u(\BF z, T) \leq T$ with probability one, because
\bex
	\Regret \left ( \BF z, T, \textsc{UE} \right)
	&=& \sum_{\BF u \in \CAL U_r} \Delta^{\BF u} \left(\BF z \right) \E \left[ N^\BF u(\BF z, T) \mid \BF Z = \BF z\right] 
	   \leq \sum_{\BF u \in \CAL U_r} \Delta^{\BF u} \left(\BF z \right) \min \left\{ 6 + 
	\frac{4 \alpha^2 \abs{\CAL U_r}  \log T  }{  \left( \Delta^{\BF u} \left(\BF z \right) \right)^2 } \,,\,  T \right\}\\
	&\leq& 6 \sum_{\BF u \in \CAL U_r} \Delta^{\BF u} \left(\BF z \right) 
		+  \max\{4 \alpha^2 ,1\} \, \abs{\CAL U_r} \, \sum_{\BF u \in \CAL U_r} \min\left\{ \frac{  \log T  }{  \Delta^{\BF u} \left(\BF z \right) } \,,\,  T \Delta^{\BF u} \left(\BF z \right)   \right\}~,
\eex
and the desired result follows from the fact that $\Delta^{\BF u}(\BF z) = \max_{\BF v \in \CAL U_r} \left( \BF v -\BF u \right)^\prime \BF z \leq  2 \bar{u} \norm{\BF z}$, by the Cauchy-Schwarz Inequality.

We will now establish an upper bound on the Bayes risk.  From the regret bound, it suffices to show  that for any $\BF u \in \CAL U_r$,
$$
	\E \left[ 					\min \left\{ \frac{ 
			\log T}{  \Delta^{\BF u} \left( \BF Z \right) }, T \Delta^{\BF u} \left( \BF Z \right) \right\}
 \right]  \leq (M_0 +1)\log T +  M_0 \log^2 T~.
$$
Let $q^{\BF u}(\cdot)$ denote the density function associated with the random variable $\Delta^{\BF u} \left( \BF Z \right)$.
Then,
\bex
\E \left[ 					\min \left\{ \frac{ 
			\log T }{  \Delta^{\BF u} \left( \BF Z \right) }, T \Delta^{\BF u} \left( \BF Z \right) \right\}
 \right] 
&=& \int_0^{ \sqrt{ \frac{\log T}{T}}}   \min \left\{ \frac{ \log T}{x},  Tx  \right\} q^{\BF u}(x) dx   \\
&&  \hspace{-0.2in}+ \int_{ \sqrt{ \frac{\log T}{T}}}^1  \min \left\{ \frac{ \log T}{x},  T x \right\} q^{\BF u}(x) dx + \int_1^\infty   \min \left\{ \frac{ \log T}{x},  T x\right\} q^{\BF u}(x) dx\, .
\eex

We will now proceed to bound each of the three terms on the right hand side of the above equality. 
Having assumed that
$q^{\BF u}(\cdot) \leq M_0$, the first term satisfies
$$
	\int_0^{  \sqrt{ (\log T) /T}}   \min \left\{ \frac{ \log T}{x},  Tx  \right\} q^{\BF u}(x) dx
	\leq M_0 	\int_0^{ \sqrt{ (\log T) /T}} Tx  \, dx = M_0T    \frac{x^2}{2} ~\Big|_0^{\sqrt{ (\log T) /T}} \leq   M_0  \log T \, .
$$
For the second term, note that 
\bex
				\int_{\sqrt{ (\log T) /T}}^1  \min \left\{ \frac{ \log T}{x},  T x \right\} q^{\BF u}(x) dx 
				&\leq& M_0  \int_{ \sqrt{ (\log T) /T}}^1  \frac{ \log T}{x} dx  =  M_0 \log T   \cdot \left(  \log x ~\Big|_{\sqrt{(\log T)/T}}^1 \right)\\
		&=&  M_0 \left(\log T \right) \cdot \frac{ \log T  - \log \log T }{2} \leq M_0 \log^2 T~,
\eex
where the last inequality follows from the fact that $\log T - \log \log T \leq 2 \log T$ for all $T \geq 2$.
To evaluate the last term, 
note that $\frac{\log T}{x} \leq \log T$ for all $x \geq 1$, and thus,
$
\int_1^\infty   \min \left\{ \frac{ \log T}{x},  T x\right\} q^{\BF u}(x) dx
\leq \log T \int_1^\infty q^{\BF u}(x) \leq \log T~.
$
Putting everything together, we have that
$
		{\E \left[ 					\min \left\{ \frac{ 
			\log T }{  \Delta^{\BF u} \left( \BF Z \right) } \,,\, T \Delta^{\BF u} \left( \BF Z \right) \right\}
 \right]}
 	\leq  (M_0+1) \log T +  M_0 \log^2 T
$,
which is the desired result.
\end{proof}

We conclude this section by giving an example of a random vector $\BF Z$ that satisfies the
condition in Theorem~\ref{thm:regret-and-risk-small-arms}.  A similar example also
appears in Example 2 of \cite{Lai:1987}.

\begin{example}[IID Random Variables]{\rm Suppose $\CAL U_r = \left\{ \BF e_1, \ldots, \BF e_r \right\}$
and
$\BF Z = \left( Z_1, \ldots, Z_r \right)$, where the random variables $Z_k$ are independent and identically distributed
with a common cumulative distribution function $F$ and  a
density function $f: \Real \to \Real$ which is bounded above by $M$. 
Then, for each $k$, the random variable $\Delta^{\BF e_k} \left( \BF Z \right)$ is given by
$
	\Delta^{\BF e_k} \left( \BF Z \right) = \left( \max_{j=1,\ldots, r} Z_j  \right) - Z_k =  \max \left\{ 0 ~,~ \max_{j \ne k} \left\{ Z_j - Z_k \right\} 
	\right\}~.
$
It is easy to verify that $\Delta^{\BF e_k} \left( \BF Z \right)$
has a point mass at $0$ and  a continuous density function $q_k(\cdot)$ on $\Real_+$ given by: for any $x > 0$,
$$
	q_k(x) =  (r-1) \int  \left\{ F( z_k + x) \right\}^{r-2}  f( z_k + x) f(z_k) dz_k
		~\leq~ (r-1) M~.
$$
}
\end{example}

\pr{
\subsection{Regret Bounds for Polyhedral Sets of Arms}

In this section, we focus on the regret profiles when the set of arms $\CAL U_r$  is a  polyhedral set. 
Let $\CAL E (\CAL U_r)$ denote the set of extreme points of $\CAL U_r$.
\jnt{From a standard result in linear programming,  for all $\BF z \in \Real^r$,}
$$
		\max_{\BF u \in \CAL U_r} \BF u^\prime \BF z = \max_{\BF u \;\in\; \CAL E(\CAL U_r)} \BF u^\prime \BF z ~.
$$
Since a polyhedral set has a finite number of extreme points ($\abs{ \CAL E (\CAL U_r) }<\infty$), 
\jnt{the parameterized bandit problem} can be reduced to the standard multi-armed bandit problem, where each arm corresponds to an extreme point of $\CAL U_r$.  We can thus apply the algorithm of \cite{LaiRobbins:1985} and obtain the following upper bound on the $T$-period cumulative regret for \jnt{polyhedra}
\beq
	\Regret\left( \BF z, T, \textrm{Lai's Algorithm} \right) 
	= O\left(  \frac{ \abs{ \CAL E (\CAL U_r) } \cdot \log T }
						{ \min
								\left\{ \Delta^{\BF u}(\BF z)  : \Delta^{\BF u}(\BF z)  > 0 \right\} } \right)~,\label{eq:lai-bound-polyhedral}
\eeq 
where the denominator corresponds to the difference between the expected reward of the optimal and the second best extreme points.
The algorithm of \cite{LaiRobbins:1985} is effective \jnt{only} when the polyhedral set $\CAL U_r$ has a small number of extreme points, as shown \jnt{by} the following example\jnt{s}.

\begin{example}[Simplex]
{\rm Suppose
$
\CAL U_r = \left\{ \BF u \in \Real^r : \sum_{i=1}^r \abs{u_i} \leq 1 \right\}
$
is an $r$-dimensional unit simplex. Then, $\CAL U_r$ has $2r$ extreme points, and  Equation (\ref{eq:lai-bound-polyhedral}) gives an $O(r \log T)$ upper bound on the regret.
}
\end{example}

\begin{example}[Linear Constraints]
{\rm Suppose that
$
\CAL U_r = \left\{ \BF u \in \Real^r:  \BF A \BF u \leq \BF b \textrm{ and } \BF u \geq \BF 0 \right\}
$,
where $\BF A$ is a \jnt{$p \times r$} matrix with $p \leq r$.  It follows from the standard linear programming theory that
every extreme point is a basic feasible solution, which has at most $p$ nonzero coordinates (see, for example, \citealp{BertsimasTsitsiklis:1997}).  Thus, the number \jnt{of} extreme points is bounded above by ${r+p \choose p} = \jnt{O((2r)^p)}$, and  Equation (\ref{eq:lai-bound-polyhedral}) gives an \jnt{$O( (2r)^p \log T)$} upper bound on the regret.
}
\end{example}

In general, the number of extreme points of a \jnt{polyhedron} can be very large, rendering the bandit algorithm of \cite{LaiRobbins:1985} ineffective; consider, for example, the $r$-dimensional cube $\CAL U_r = \left\{ \BF u \in \Real^r : \abs{u_i} \leq 1 \textrm{ for all } i \right\}$, which  has $2^r$ extreme points.  Moreover, we cannot apply the results and algorithms from Section \ref{section:matching-upper-bound} to the convex hull of $\CAL U_r$.  This is because the convex hull of a polyhedron is {\em not} strongly convex (it cannot be written as an intersection of Euclidean balls), and thus, it does not satisfy the required SBAR$(\cdot)$ condition in Theorem~\ref{theorem:greedy-regret}.   The \textsc{UE} policy in the previous section \jnt{gives} $O(r \sqrt{T} \log^{3/2} T)$ regret and risk upper bounds.  However, finding an algorithm specifically for polyhedral sets that yields an $O( r \sqrt{T} )$ regret upper bound (without an additional logarithmic factor) remains an open question.  
}

\section{Conclusion} \label{section:conclusion}

We analyzed a class of multiarmed bandit problems where the expected reward of
each arm depends linearly on an unobserved random vector $\BF Z \in \Real^r$, with $r \geq 2$. 
Our model allows for correlations among the rewards of different arms. 
When we have a smooth best arm response, we showed that 
a policy that alternates between exploration and exploitation is optimal. For a general bandit,
we proposed a near-optimal policy that performs active exploration
in every period.  For finitely many arms,
our policy achieves asymptotically optimal regret and risk
as a function of time, but scales with the square of the number of arms.
Improving the dependence on the number of arms remains an open question.
It would also be interesting
to study more general correlation structures.
Our formulation assumes that the vector of expected rewards lies in an $r$-dimensional
subspace spanned by a known set of basis functions that describe the characteristics
of the arms.  Extending our work to a setting where the basis functions are unknown has the potential to broaden the applicability of our model.

{\section*{Acknowledgement}
The authors would like to thank 
Adrian Lewis and Mike Todd for helpful and stimulating
discussions on the structure of positive definite matrices and the eigenvalues associated with least squares estimators, 
Gena Samorodnitsky for sharing his deep insights on the application of large deviation theory
to this problem,  Adam Mersereau for his contributions to the problem formulation, and Assaf Zeevi for helpful suggestions and discussions during the first author's visit to Columbia Graduate School of Business.  We also want to thank the Associate Editor and the referee for their helpful comments and suggestions on the paper.
This research is supported in part by the National Science Foundation through grants DMS-0732196, ECCS-0701623, CMMI-0856063, and CMMI-0855928.}


{\fontsize{9}{9}\selectfont{
\bibliographystyle{wsc}
\bibliography{references}}}

\appendix
\section{Properties of Normal Vectors} 
\label{section:proof-guassian-properties}

In this section, we prove that  if $\BF Z$ has  a multivariate normal distribution with mean $\BF 0 \in \Real^r$ and covariance matrix $\BF I_r / r$, then $\BF Z$ has the properties described in
Lemmas \ref{lemma:gaussian-lower-bound} and \ref{lemma:gaussian-has-small-mass}.

\subsection{Proof of Lemma  \ref{lemma:gaussian-lower-bound}}
\label{section:proof-gaussian-lower-bound}

We want to establish a lower bound on  $\Pr \left\{  \theta  ~\leq~ \norm{\BF Z} ~\leq~ \beta  \right\}$.   Let $\BF Y  = \left( Y_1, \ldots, Y_r \right)$ denote the standard multivariate normal random vector with mean $\BF 0$ and identity covariance matrix $\BF I_r$.  By our hypothesis, $\BF Z$ has the same distribution as $\BF Y / \sqrt{r}$, which implies that
$$ 
	\Pr \left\{  \theta  \,\leq\, \norm{\BF Z} \,\leq\, \beta  \right\} 
	=\Pr \left\{ \theta \sqrt{r}  \,\leq\, \norm{\BF Y} \,\leq\, \beta \sqrt{r} \right\} = 1  - \Pr \left\{  \norm{\BF Y}^2  < \theta^2 r \right\} -  \Pr \left\{  \norm{\BF Y}^2 > \beta^2 r \right\}~.
$$
By definition,
$\norm{\BF Y}^2 = Y_1^2 + \cdots + Y_r^2$ has a chi-square distribution with $r$ degrees of freedom.
By the Markov Inequality, 
$\Pr \left\{  \norm{\BF Y}^2  > \beta^2 r \right\}
	\leq \E \left[ \norm{\BF Y}^2 \right] / (\beta^2 r)  = 1/ \beta^2$.
We will now establish an upper bound on $\Pr \left\{  \norm{\BF Y}^2  < \theta^2 r \right\}$.
Note that, for any $\lambda > 0$,
$$
	\Pr \left\{  \norm{\BF Y}^2  < \theta^2 r \right\}
	= 	\Pr \left\{ e^{  - \lambda \sum_{k=1}^r Y_k^2}  >  e^{- \lambda \theta^2 r} \right\}
	\leq e^{ \lambda \theta^2 r} \cdot \E \left[ \prod_{k=1}^r e^{ - \lambda Y_k^2 } \right]
		= \left( \frac{ e^{\lambda \theta^2} }{ \sqrt{ 1 + 2 \lambda} } \right)^r~,
$$
where last equality follows from the fact that $Y_1, \ldots, Y_r$ are independent standard normal random variables and  thus, $\E \left[ e^{-\lambda Y_k^2} \right] = 1 / \sqrt{ 1 + 2 \lambda}$ for $\lambda > 0$.  Set $\lambda = 1 / \theta^2$, and use the facts $\theta \leq 1/2 \leq \sqrt{2}/e$
and $r \geq 2$, to obtain
$$
	\Pr \left\{  \norm{\BF Y}^2  < \theta^2 r \right\}
		\leq \left( \frac{ e \theta}{\sqrt{ 2 + \theta^2}} \right)^r 
		\leq \left( \frac{ e \theta}{\sqrt{ 2 }} \right)^r  \leq  \left( \frac{ e \theta}{\sqrt{ 2 }} \right)^2 = 
		\frac{ e^2 \theta^2}{ 2 } \leq    4 \theta^2~,
$$
which implies that  
$\Pr \left\{  \theta  \,\leq\, \norm{\BF Z} \,\leq\, \beta  \right\} 
\geq 1 - \frac{ 1}{\beta^2} - 4 \theta^2$, which is the desired result.

\subsection{Proof of Lemma \ref{lemma:gaussian-has-small-mass}}
\label{section:proof-gaussian-has-small-mass}


For part (a) of the lemma, we have
$$
	\E \left[ \, 1/\norm{\BF Z} \, \right]	
	~=~ \int_0^{\infty}   \frac{1}{x} g(x)\, dx 
	~\leq~ M_0\int_0^\rho x^{\rho-1}\, dx+ \frac{1}{\rho} \int_{\rho}^{\infty} g(x)\, dx
	~\leq~ M_0\frac{\rho^{\rho}}{\rho}+ \frac{1}{\rho}.
$$
For the proof of part (b), let $\BF Y  = \left( Y_1, \ldots, Y_r \right)$ be a standard multivariate normal random vector with mean $\BF 0$ and identity covariance matrix, $\BF I_r$. Then, $\BF Z$ has the same distribution as $\BF Y / \sqrt{r}$.  Note that $\norm{\BF Y}^2$ has a chi-square distribution with $r$ degrees of freedom. Thus, 
$$\E[\, \norm{\BF Z}\, ] = \frac{1}{\sqrt{r}} \E[\, \norm{\BF Y}\, ]
\leq \frac{1}{\sqrt{r}} \sqrt{ \E[\, \norm{\BF Y}^2\, ]}
=\frac{1}{\sqrt{r}}\sqrt{r}=1~.
$$

We will now establish an upper bound on $\E[\, 1/\norm{\BF Z}\, ] =
\sqrt{r} \, \E[\, 1/\norm{\BF Y}\, ]$.  For $r=2$, since $\norm{\BF Y}$ has a chi distribution with two degrees of freedom, we have that
$$
	\E[\, 1/\norm{\BF Z}\, ] =
	\sqrt{2} \, \int_0^\infty \frac{1}{x} \cdot x e^{-x^2/2} \, dx
	= \sqrt{2} \, \int_0^\infty e^{-x^2/2} \, dx  = \sqrt{ \pi}~.
$$
Consider the case where $r \geq 3$. Then,
$$
\E[\, 1/\norm{\BF Z}\, ] =
\sqrt{r} \, \E[\, 1/\norm{\BF Y}\, ] 
\leq \sqrt{r}\,  \sqrt{ \E[\, 1/\norm{\BF Y}^2\, ]}~.
$$
Using the formula for the density of the chi-square distribution, we have
\begin{eqnarray*}
\E[\, 1/\norm{\BF Y}^2\, ]
&=&  \int_0^{\infty} \frac{1}{x}\cdot \frac{1}{2^{r/2} \Gamma(r/2)}
x^{(r/2)-1} e^{-x/2}\, dx\\
&=&  \frac{2^{(r/2)-1}}{2^{r/2}}\cdot 
\frac{\Gamma((r/2)-1)}{\Gamma(r/2)}
\cdot
 \int_0^{\infty}  \frac{1}{2^{(r-2)/2} \Gamma((r-2)/2)}
x^{((r-2)/2)-1} e^{-x/2}\, dx \\
&=&  \frac{1}{ 2 ( (r/2) - 1)} = \frac{1}{r - 2} \leq \frac{3}{r}~, 
\end{eqnarray*}
where the third equality follows from the fact that
$\Gamma(r/2)=\left( (r/2)  - 1 \right) \cdot \Gamma ((r/2)-1)$ for $r \geq 3$
and 
the integrand is the density function of the chi-square distribution with $r-2$ degrees of freedom and evaluates to $1$.  The 
last inequality follows because $r \geq 3$. Thus, we have 
$\E[\, 1/\norm{\BF Z}\, ] \leq \sqrt{3} \leq \sqrt{\pi}$, which is the desired result.

\newpage  

\section{Proof of Theorems \ref{theorem:UE-regret-analysis} and
\ref{thm:regret-and-risk-small-arms}
}
\label{appendix:UE-proofs}

In the next section, we 
establish large deviation inequalities for adaptive least
squares estimators (with unbounded error random variables), which will be used
in the proof of Theorems \ref{theorem:UE-regret-analysis} and
\ref{thm:regret-and-risk-small-arms}, given in Sections
\ref{appendix:proof-UE-regret-analysis}
and \ref{appendix:proof-regret-and-risk-small-arms}, respectively. 

\subsection{Large Deviation Inequalities}
\label{appendix:large-deviation-inq}

The first result extend  the standard Chernoff Inequality to
our setting involving uncertainty ellipsoids when we have finitely many arms. 

\begin{thm}[Chernoff Inequality for Uncertainty Ellipsoids with Finitely Many Arms] \label{thm:chernoff-inq-for-UE-finite-arms}   
Under Assumption \ref{assumption:basic},
for any $t \geq r$, $\BF x \in \Real^r$, 
$\BF z \in \Real^r$, and $\zeta > 0$,
$$
	\Pr \left\{  \BF x^\prime \left( \widehat{\BF Z}_t - \BF z \right)   
			~>~ \zeta \,  \sigma_0  \,   \norm{ \BF x}_{\BF C_t}  ~\big|~ \BF Z = \BF z \right\} ~\leq~    t^{5\abs{\CAL U_r}}  e^{-\zeta^2/2},
$$
and 
$$
	\Pr \left\{   \left( \BF U_{t+1} - \BF x \right)^\prime \left( \widehat{\BF Z}_t - \BF z \right) 
	~>~  \zeta \,  \sigma_0   \norm{ \BF U_{t+1} - \BF x }_{\BF C_t}  ~\big|~ \BF Z = \BF z \right\} ~\leq~  t^{5\abs{\CAL U_r}}  e^{-\zeta^2/2}~.
$$
\end{thm}
\begin{proof}  We will only prove the second inequality because the proof of the first one follows the same argument.
If the sequence of arms $\BF U_1, \BF U_2, \ldots$ is deterministic (and thus, the matrix $\BF C_t$ is also deterministic), then 
$$
	\frac{ \left( \BF U_{t+1} - \BF x \right)^\prime \left( \widehat{\BF Z}_t - \BF z \right) }{ \norm{ \BF U_{t+1} - \BF x }_{\BF C_t} }
	= \sum_{s=1}^t \frac{ \left( \BF U_{t+1} - \BF x \right)^\prime \BF C_t \BF U_s }{ \norm{\BF U_{t+1} - \BF x}_{\BF C_t}} W_s
$$
and
$$
	\sum_{s=1}^t \left( \frac{ \left( \BF U_{t+1} - \BF x \right)^\prime \BF C_t \BF U_s }{ \norm{ \BF U_{t+1} - \BF x }_{\BF C_t}}\right)^2
	= \frac{  \left( \BF U_{t+1} - \BF x \right)^\prime \BF C_t \left( \sum_{s=1}^t \BF U_s \BF U_s^\prime \right) \BF C_t \left( \BF U_{t+1} - \BF x \right)}{ \left( \BF U_{t+1} - \BF x \right)^\prime \BF C_t \left( \BF U_{t+1} - \BF x \right)} = 1~.
$$
The classical Chernoff Inequality for the sum of independent random variables (see, for example, Chapter 1 in \citealp{Dudley:1999}) then yields
\bex
	\lefteqn{ \Pr \left\{  \left( \BF U_{t+1} - \BF x \right)^\prime \left( \widehat{\BF Z}_t - \BF z \right)   
			~>~ \zeta \,  \sigma_0  \,   \norm{  \BF U_{t+1} - \BF x }_{\BF C_t}  ~\big|~ \BF Z = \BF z \right\} }\\
	&\leq& \exp \left\{ -  \frac{ \zeta^2  \sigma_0^2 }{ 2 \sigma_0^2  
	\sum_{s=1}^t \left( \left( \BF U_{t+1} - \BF x \right)^\prime \BF C_t \BF U_s ~\big/~ \norm{ \BF U_{t+1} - \BF x }_{\BF C_t} \right)^2  } \right\} = e^{-\zeta^2 / 2}~.
\eex
In our setting, however, the arms $\BF U_t$ are random variables that depend on the accumulated history, and
we cannot apply the standard Chernoff inequality directly.  If $N^{\BF u}(\BF z, t)$ denotes the total number of times that  arm $\BF u$ has been chosen during the first $t$ periods given that $\BF Z = \BF z$, then
$$
		\BF C_t = \left( \sum_{s=1}^t \BF U_s \BF U_s^\prime \right)^{-1} 
				   = \left( \sum_{\BF u \in \CAL U_r} N^{\BF u} (\BF z, t) \, \BF u \BF u^\prime  \right)^{-1}~,
$$
which shows that the  matrix $\BF C_t$ is completely determined by the nonnegative integer random variables
$N^{\BF u}(\BF z, t)$.  
Since $0\leq N^{\BF u}(\BF z, t) \leq t$, 
the number of possible
values of the vector  $\left( N^{\BF u}(\BF z, t) : \BF u \in \CAL U_r \right)$ is at most $t^{ \abs{\CAL U_r}}$. It then follows easily that
the number of different values of the ordered pair $\left(\BF U_{t+1}, \BF C_t \right)$
is at most 
$\abs{\CAL U_r} t^{ \abs{\CAL U_r}} 
 \leq t^{ 5 \abs{\CAL U_r}}
$. 
 To get the desired result, we can then use the union bound and apply the classical Chernoff Inequality to each ordered pair.
\end{proof}

When the number of arms is infinite, the bounds in Theorem \ref{thm:chernoff-inq-for-UE-finite-arms} are vacuous. The following theorem provides an extension of the Chernoff inequality to the case of 
 infinitely many arms.  

\begin{thm}[Chernoff Inequality for Uncertainty Ellipsoids with Infinitely Many Arms] \label{thm:chernoff-inq-for-UE-infinite-arms}   
Under Assumption \ref{assumption:basic},
for any $t \geq r$, $\BF x \in \Real^r$, 
$\BF z \in \Real^r$, and $\zeta \geq 2$,
$$
	\Pr \left\{  \BF x^\prime \left( \widehat{\BF Z}_t - \BF z \right)   
			~>~ \zeta \, \kappa_0  \,  \sigma_0  \,  \sqrt{ \log t}   \norm{ \BF x}_{\BF C_t}  ~\big|~ \BF Z = \BF z \right\} ~\leq~    t^{ r \, \kappa_0^2  } \; e^{ -\, \zeta^2 / 4 }~, 
$$
and 
$$
	\Pr \left\{   \left( \BF U_{t+1} - \BF x \right)^\prime \left( \widehat{\BF Z}_t - \BF z \right) 
	~>~  \zeta \, \kappa_0  \,  \sigma_0  \,  \sqrt{ \log t} \norm{ \BF U_{t+1} - \BF x }_{\BF C_t}  ~\big|~ \BF Z = \BF z \right\} ~\leq~    t^{r  \, \kappa_0^2  } \; e^{ -\, \zeta^2 / 4 }~.
$$
\end{thm}

The proof of Theorem \ref{thm:chernoff-inq-for-UE-infinite-arms} makes use of the following series of lemmas.  The first lemma establishes a tail inequality for a ratio of two random variables. 
\cite{DeLaPenaKL:2004} gave a proof of this result in 
 Corollary 2.2 (page 1908) of their paper .

\begin{lem}[Exponential Inequality for Ratios, \citealp{DeLaPenaKL:2004}] \label{lemma:tail-bound-for-ratios} Let $A$ and $B$ be two random variables such that $B \geq 0$
with probability one and $\E \left[ e^{ \gamma A -  \left(\gamma^2 B^2 / 2 \right) } \right] \leq 1$
for all $\gamma \in \Real$. Then, for all $\zeta \geq \sqrt{2}$ and $y > 0$,
$$
	\Pr \left\{  \abs{A} 
	 \geq \zeta  ~ \sqrt{  \left( B^2 + y \right) \left( 1 + \frac{1}{2} \log  \left( 1 +  \frac{B^2}{ y} \right) \right)}
	 \right\} ~\leq~ e^{ - \zeta^2 / 2}
$$
\end{lem}

Recall from Equation (\ref{eq:OLS}) that $\BF M_t = \sum_{s=1}^t \BF U_s W_s$
is the martingale associated with the least squares estimate $\widehat{\BF Z}_t$.  The next lemma establishes a martingale inequality associated with the inner product 
$\BF x^\prime \BF M_t$ for an arbitrary vector  $\BF x \in \Real^r$.  This result is based on  Lemma~\ref{lemma:tail-bound-for-ratios} with
$A = \frac{\BF x^\prime \BF M_t}{\sigma_0}  = \sum_{s=1}^t \frac{ \left( \BF x^\prime \BF U_s \right)}{\sigma_0} W_s$ and 
$B = \norm{ \BF x}_{\BF C_t^{-1}} = \sqrt{ \BF x^\prime \BF C_t^{-1} \BF x}  = \sqrt{ \sum_{s=1}^t \left( \BF x^\prime \BF U_s  \right)^2}$.  We then use upper and lower bounds on $B^2$
to establish bounds on the term $\log  \left( 1 +  \frac{B^2}{ y} \right)$, 
for a suitable choice of $y$, giving us the desired result.

\begin{lem}[Martingale Inequality] \label{lemma:martingale-inq} Under Assumption \ref{assumption:basic}, for any $\BF x \in \Real^r$, $ t \geq 1$,  and $\zeta \geq \sqrt{2}$,
$$
\Pr \left\{  \abs{\BF x^\prime \BF M_t}  \;>\;  \zeta \, \kappa_0  \, \sigma_0 \sqrt{\log t} 
 	\norm{ \BF x}_{\BF C_t^{-1}} 
	\right\} = \Pr \left\{  \BF x^\prime \BF M_t \BF M_t^\prime \BF x  \;>\;  \zeta^2 \, \kappa_0^2 \, \sigma_0^2  (\log t) 
 	\left( \BF x^\prime \BF C_t^{-1} \BF x \right) 
	\right\} ~\leq~ e^{ -\zeta^2 / 2}~.
$$
\end{lem}
\begin{proof}  Let $\BF x \in \Real^r$ and $t \geq 1$ be given.   Without loss of generality, we can assume that $\norm{\BF x} = 1$.  Let the random variables $A$ and $B$
be defined by
$$
	A = \frac{\BF x^\prime \BF M_t}{\sigma_0}  = \sum_{s=1}^t \frac{ \left( \BF x^\prime \BF U_s \right)}{\sigma_0} W_s
	\quad \textrm{and} \quad
	B = \norm{ \BF x}_{\BF C_t^{-1}} = \sqrt{ \BF x^\prime \BF C_t^{-1} \BF x} = \sqrt{ \sum_{s=1}^t \left( \BF x^\prime \BF U_s  \right)^2}~.
$$
For any $s$, let $\BF H_s = \left( \BF U_1, X_1, W_1, \ldots, \BF U_s, X_s, W_s \right)$
the history until the end of period $s$.
By definition, 
$\BF U_{s}$ is a function of $\BF H_{s-1}$, and 
it follows from 
Assumption \ref{assumption:basic}(a)
that for any $\gamma \in \Real$,
$$
 \E \left[ 
	 e^{   \frac{ \gamma}{\sigma_0} \left( \BF x^\prime \BF U_s \right) W_s
						- \frac{ \gamma^2 \left( \BF x^\prime \BF U_s  \right)^2 }{2} }	 
	 ~\Big|~ \BF H_{s-1} \right]
=  e^{  - \frac{ \gamma^2 \left( \BF x^\prime \BF U_s  \right)^2 }{2} }
 \E \left[ 
	 e^{   \frac{ \gamma}{\sigma_0} \left( \BF x^\prime \BF U_s \right) W_s }	 
	 ~\Big|~ \BF H_{s-1} \right] \leq 1~.
$$
Using a standard argument involving iterated expectations, we obtain
$$
	\E \left[ e^{ \gamma A -  \left( \gamma^2 B^2 / 2 \right) }  \right]
	= 	\E \left[ e^{ 
					\sum_{s=1}^t 
						\left(   \frac{ \gamma}{\sigma_0} \left( \BF x^\prime \BF U_s \right)  W_s
						- \frac{ \gamma^2 \left( \BF x^\prime \BF U_s  \right)^2 }{2} \right)
	 }  \right] = \E \left[ \prod_{s=1}^t e^{ 
						\left( \frac{ \gamma}{\sigma_0} \left( \BF x^\prime \BF U_s \right)  W_s
						- \frac{ \gamma^2 \left( \BF x^\prime \BF U_s  \right)^2 }{2} \right)
	 }  \right] \leq 1
$$

We can thus apply Lemma \ref{lemma:tail-bound-for-ratios} to the random variables $A$ and $B$. Moreover,  it follows from the definition of $\bar{u}$ and $\lambda_0$ in Assumption \ref{assumption:basic}(b) that, with probability one,  
$$
\lambda_0 \leq  \lambda_{\min} \left(\sum_{s=1}^t \BF U_s \BF U_s^\prime \right) \leq 
\BF x^\prime \left( \sum_{s=1}^t  \BF U_s \BF U_s^\prime \right) \BF x
= B^2  = \sum_{s=1}^t \left( \BF x^\prime \BF U_s \right)^2 \leq t \bar{u}^2~.
$$
Therefore,  $ B^2 + \lambda_0   \leq  2 B^2$, and
\bex
	1 + \frac{1}{2} \log  \left( 1 +  \frac{B^2}{\lambda_0} \right)  
&\leq& 1 + \frac{1}{2} \log  \left( 1 +  \frac{ t \bar{u}^2 }{\lambda_0} \right)  
\leq \frac{1}{2} \left( \log t +  2+ \log \left( 1 +  \frac{  \bar{u}^2 }{\lambda_0} \right) \right)\\
&\leq&  \frac{ \log t}{2} \left( 1 + \frac{ 2 + \log \left( 1 +  (\bar{u}^2 / \lambda_0) \right) }{ \log t}  \right) \leq \frac{ \kappa_0^2 \log t}{2 } ~,
\eex
where the last inequality follows from the definition of $\kappa_0$ and the fact that $t \geq r \geq 2$.
These two upper bounds imply  that
$
\sqrt{ \left( B^2 + \lambda_0 \right) \left( 1 + \frac{1}{2} \log  \left( 1 +  \frac{B^2}{\lambda_0} \right) \right)}
	\leq  \kappa_0  \sqrt{ \log t } \; B~.
$
Therefore,
\bex
\Pr \left\{  \abs{\BF x^\prime \BF M_t}  >  \zeta \, \kappa_0 \, \sigma_0 \sqrt{\log t} 
 	\norm{ \BF x}_{\BF C_t^{-1}} 
	\right\}  &\leq& \Pr \left\{ \abs{A} > \zeta  \sqrt{  \left( B^2 + \lambda_0 \right) \left( 1 + \frac{1}{2} \log  \left( 1 +  \frac{B^2}{\lambda_0} \right) \right)} 
 \right\}~, 
\eex
and the desired result then follows immediately from Lemma \ref{lemma:tail-bound-for-ratios}.
\end{proof}

The next and final lemma extends the previous result to show that the matrix 
$
\zeta^2 \, \kappa_0^2 \, \sigma_0^2 \, (\log t)  \, \BF C_t^{-1}  - \BF M_t \BF M_t^\prime$
is positive semidefinite with a high probability.  The proof of this result makes use of
the fact that for the matrix $\zeta^2 \, \kappa_0^2 \, \sigma_0^2 \,(\log t)  \, \BF C_t^{-1}  - \BF M_t \BF M_t^\prime$ to be positive semidefinite,
it suffices for the inequality 
$ \BF x^\prime \BF M_t \BF M_t^\prime \BF x \leq \zeta^2 \, \kappa_0^2 \, \sigma_0^2 \, (\log t)  \, \BF x^\prime  \BF C_t^{-1}  \BF x  $ 
to hold for vectors $\BF x$ in a sufficiently dense subset.   We can then apply Lemma \ref{lemma:martingale-inq} for each such vector $\BF x$ and use the union bound.

\begin{lem} \label{lemma:positive-definite-matrix-tail} Under Assumption \ref{assumption:basic}, for any $t \geq r$ and $\zeta \geq 2$,
$$
\Pr \left\{ \BF M_t \BF M_t^\prime
	~\leq~ \zeta^2 \, \kappa_0^2 \, \sigma_0^2 \, (\log t)  \, \BF C_t^{-1} \right\}  ~\geq~  
		1 \,-\,  t^{r \, \kappa_0^2  } \; e^{ -\, \zeta^2 / 4 }~,
$$
\end{lem}
\begin{proof}
Let $\CAL S^r = \left\{ \BF x \in \Real^r : \norm{\BF x} = 1 \right\}$ denote the unit sphere in $\Real^r$.
Let $\delta > 0$ be defined by:
$$
			\delta = \frac{ \lambda_0}{ 9 \bar{u}^2 t}~,
$$
where the constants $\lambda_0$ and $\bar{u}$ are given in Assumption \ref{assumption:basic}(b).  Without loss of generality, we can assume that $\delta \leq 1/2$
and that $1/\delta$ is an integer.  Let $\CAL X^r$ be a covering of $\CAL S^r$, that is, for any $\BF x \in \CAL S^r$, there exists $\BF y \in \CAL X^r$ such that $\norm{\BF x - \BF y} \leq \delta$.  It is easy to verify that $\CAL X^r$ can be chosen to have a cardinality of at most $\left(2  \sqrt{r}  / \delta \right)^r$ because we can consider a rectangular grid on $[-1, 1]^r$ with a grid spacing of $\delta / \sqrt{r}$.   Then, for any point $\BF x \in \CAL S^r$, there is a point $\BF y$ on the rectangular grid  such that the magnitude of each component of $\BF x - \BF y$ is at most $\delta/ \sqrt{r}$, which implies that $\norm{\BF x - \BF y} \leq \delta$.

Let $t \geq r$ and $\zeta \geq 2$ be given.  To facilitate our exposition, let $\beta = \zeta^2 \,\kappa_0^2 \, \sigma_0^2 \, \log t$.
Let $\CAL G$ denote the event that the following inequalities hold:
$$
	\BF e_i^\prime  \BF M_t \BF M_t^\prime \BF e_i \leq \beta \BF e_i \BF C_t^{-1} \BF e_i, 
	\quad i = 1, 2, \ldots, r~,
	\quad \textrm{and} \quad
	\BF y^\prime \BF M_t \BF M_t^\prime \BF y \leq 
	\frac{ \beta }{2} \BF y^\prime \BF C_t^{-1} \BF y~,
	\quad \forall ~\BF y \in \CAL X^r~.
$$
Using the  union bound, it follows from Lemma \ref{lemma:martingale-inq} that 
the event $\CAL G$ happens with a probability at least 
\bex
	 1 - \abs{\CAL X^r} e^{- \zeta^2 /4 }  - r e^{ - \zeta^2 / 2 } &\geq& 1 - 
	 	\left( \frac{ 2 \sqrt{r} }{\delta} \right)^r e^{- \zeta^2 /4 }  - r e^{ - \zeta^2 / 2 }
		~\geq~ 1 - \left( \left( \frac{ 2 \sqrt{r} }{\delta} \right)^r + r \right) e^{ - \zeta^2 / 4} \\
		&\geq& 1 -  \left( \frac{ 4 \sqrt{r} }{\delta} \right)^r  e^{ - \zeta^2 / 4}   
		~\geq~ 1 - \left( \frac{ 36 \, \bar{u}^2 \, t^{2} }{\lambda_0 } \right)^r  e^{ - \zeta^2 / 4} 
		~\geq~  1 - t^{r \, \kappa_0^2 } \; e^{ - \zeta^2 / 4}~,
\eex
where we have used the fact that $t \geq \sqrt{r} \geq 2$ in the penultimate inequality.  The
final inequality follows from the definition of $\kappa_0$ in Equation (\ref{eq:constant-bounds}), which implies that $\kappa_0^2 \geq 4 \left( 1 + \log (36 \bar{u}^2 / \lambda_0) \right) \geq 4$, and thus,
$
\frac{ 36 \, \bar{u}^2 \, t^{2} }{\lambda_0 } \leq t^2 e^{\kappa_0^2/4} 
\leq t^{ \kappa_0^2 / 2}  \left( t^2 \right)^{\kappa_0^2/4} = t^{ \kappa_0^2}~.
$

To complete the proof,  it suffices to show that when the event $\CAL G$ occurs, we have that
$\BF x^\prime \BF M_t \BF M_t^\prime \BF x \leq \beta \, \BF x^\prime \BF C_t^{-1} \BF x$ for all
$\BF x \in \CAL S^r$.  Consider an arbitrary $\BF x \in \CAL S^r$, and let $\BF y \in \CAL X^r$
be such that $\norm{\BF x - \BF y } \leq \delta$. This implies that $\norm{\BF x + \BF y} \leq 
\norm{2 \BF x} + \norm{ \BF y - \BF x} \leq 2 + \delta \leq 3$. Moreover,
$\BF x^\prime \BF M_t \BF M_t^\prime \BF x - \BF y^\prime \BF M_t \BF M_t^\prime \BF y
= \left( \BF x - \BF y \right)^\prime \BF M_t \BF M_t^\prime \left( \BF x + \BF y \right)
\leq 3 \delta \norm{\BF M_t}^2$ where we use the Cauchy-Schwarz for the last inequality. 

Similarly, we can show that for all $s$, $\BF y^\prime \BF U_s \BF U_s^\prime \BF y
\leq \BF x^\prime \BF U_s \BF U_s^\prime \BF x + 3 \delta \bar{u}^2$. Summing over all $s$, we 
obtain that $\BF y^\prime \BF C_t^{-1} \BF y \leq \BF x \BF C_t^{-1} \BF x + 3  \delta t \bar{u}^2$.
Putting everything together, we have that
\bex
	\BF x^\prime \BF M_t \BF M_t^\prime \BF x
	&\leq&  	\BF y^\prime \BF M_t \BF M_t^\prime \BF y + 3 \delta \norm{\BF M_t}^2 
	\leq  \frac{ \beta }{2} \BF y^\prime \BF C_t^{-1} \BF y + 3 \delta \norm{\BF M_t}^2\\
	&\leq&  \frac{ \beta }{2} \BF x^\prime \BF C_t^{-1} \BF x + \frac{3 \beta}{2}  \delta t \bar{u}^2
				+ 3 \delta \norm{\BF M_t}^2
	\leq \beta \BF x^\prime \BF C_t^{-1} \BF x  - \frac{ \beta }{2}\lambda_0 + \frac{3 \beta}{2}  \delta t \bar{u}^2
				+ 3 \delta \norm{\BF M_t}^2~,
\eex
where the last inequality follows from the fact that $\BF C_t^{-1} = \sum_{s=1}^t \BF U_s \BF U_s^\prime \geq \lambda_0 \BF I_r$ from our definition of $\lambda_0$.  Finally, note that under the event $
\CAL G$, 
\bex
\norm{\BF M_t}^2  &=& \sum_{i=1}^r \BF e_i^\prime \BF M_t \BF M_t^\prime \BF e_i
						\leq \beta \sum_{i=1}^r \BF e_i^\prime \BF C_t^{-1} \BF e_i 
						= \beta \sum_{s=1}^t \sum_{i=1}^r \BF e_i^\prime \BF U_s \BF U_s^\prime \BF e_i\\
						&=& \beta \sum_{s=1}^t \sum_{i=1}^r \abs{\BF e_i^\prime \BF U_s}^2
						= \beta \sum_{s=1}^t \norm{\BF U_s}^2 \leq \beta t \bar{u}^2~,
\eex
which implies that
$$
	- \frac{ \beta }{2}\lambda_0 + \frac{3 \beta}{2}  \delta t \bar{u}^2
				+ 3 \delta \norm{\BF M_t}^2
		\leq 	- \frac{ \beta }{2}\lambda_0 + \frac{9 \beta}{2} \delta t \bar{u}^2 
		= \frac{\beta}{2} \left( 9 \delta t \bar{u}^2 - \lambda_0 \right) = 0~,
$$
where the last inequality follows from the definition of $\delta$.  Thus, we have that
$	\BF x^\prime \BF M_t \BF M_t^\prime \BF x \leq 
\beta \BF x^\prime \BF C_t^{-1} \BF x$, which is the desired result.
\end{proof}

We are now ready to give a proof of Theorem \ref{thm:chernoff-inq-for-UE-infinite-arms}.

\begin{proof}
It suffices to establish the second inequality in Theorem \ref{thm:chernoff-inq-for-UE-infinite-arms} because
the proof for the first inequality follows  the same argument. It follows from the Cauchy-Schwarz
inequality that
$$
	 \frac{ \left( \BF U_{t+1} - \BF x \right)^\prime \left( \widehat{\BF Z}_t - \BF z \right) }{
	 \norm{ \BF U_{t+1} - \BF x }_{\BF C_t} }
	 =
	  	\frac{ \left( \BF U_{t+1} - \BF x \right)^\prime \BF C_t^{1/2} \BF C_t^{-1/2} \left( \widehat{\BF Z}_t - \BF z \right)}{ \norm{\BF C_t^{1/2} \left( \BF U_{t+1} - \BF x \right) } } 
	\leq \norm{\BF C_t^{-1/2} \left( \widehat{\BF Z}_t - \BF z \right)}~,
$$
with probability one. Therefore,
\bex
	\lefteqn{ \Pr \left\{   \left( \BF U_{t+1} - \BF x \right)^\prime \left( \widehat{\BF Z}_t - \BF z \right) 
	~>~ \zeta \, \kappa_0 \, \sigma_0 \, \sqrt{\log t} \norm{ \BF U_{t+1} - \BF x }_{\BF C_t}  ~\big|~ \BF Z = \BF z \right\} }\\
		&\leq& 	\Pr \left\{  \norm{\BF C_t^{-1/2} \left( \widehat{\BF Z}_t - \BF z \right)} 
	~>~ \zeta \, \kappa_0 \, \sigma_0 \, \sqrt{\log t}  ~\big|~ \BF Z = \BF z \right\} \\
		&=& \Pr \left\{ \left( \widehat{\BF Z}_t - \BF z \right)^\prime \BF C_t^{-1} \left( \widehat{\BF Z}_t - \BF z \right) ~>~ \zeta^2 \, \kappa_0^2 \, \sigma_0^2 \, \log t  ~\big|~ \BF Z = \BF z \right\} \\
		&=& \Pr \left\{ \BF M_t^\prime \BF C_t \BF M_t ~>~ \zeta^2 \, \kappa_0^2 \, \sigma_0^2 \, \log t  ~\big|~ \BF Z = \BF z \right\}~,
\eex
where the last equality follows from the definition of the least squares estimate $\widehat{\BF Z}_t$.

It is a well-known result in linear algebra (see, for example, Theorem 1.3.3 in \citealp{Bhatia:2007}) that if $\BF A$ and $\BF B$ are two symmetric positive definite matrices, then the block matrix
$
\left(
\begin{array}{cc}
\BF A & \BF X  \\ \BF X^\prime & \BF B  
\end{array}
\right)
$ is positive semidefinite {\em if and only if } $ \BF X \BF B^{-1} \BF X^\prime \leq \BF A$.  Applying
this result to the two ``equivalent'' $(r+1) \times (r+1)$ matrices 
$
\left(
\begin{array}{cc}
\zeta^2 \, \kappa_0^2 \, \sigma_0^2 \, \log t & \BF M_t^{\prime}  \\ 
\BF M_t  & \BF C_t^{-1}
\end{array}
\right)
\quad \textrm{and} \quad
\left(
\begin{array}{cc}
\BF C_t^{-1} & \BF M_t  \\ 
\BF M_t^\prime  & \zeta^2 \, \kappa_0^2 \, \sigma_0^2 \,\log t
\end{array}
\right)~,
$
we conclude that  $\BF M_t^\prime \BF C_t \BF M_t \leq \zeta^2 \, \kappa_0^2 \, \sigma_0^2 \, \log t$ if and only if $\BF M_t \BF M_t^\prime \leq \zeta^2 \, \kappa_0^2 \, \sigma_0^2 \, (\log t) \BF C_t^{-1}$. The desired
result then follows from the fact  that
\bex
	\Pr \left\{ \BF M_t^\prime \BF C_t \BF M_t ~>~ \zeta^2 \, \kappa_0^2 \, \sigma_0^2 \, \log t  ~\big|~ \BF Z = \BF z \right\}
	&=& 1 - \Pr \left\{   \BF M_t^\prime \BF C_t \BF M_t ~\leq~ \zeta^2 \, \kappa_0^2 \, \sigma_0^2 \, \log t  ~\big|~ \BF Z = \BF z \right\}\\
	&=& 1 - \Pr \left\{   \BF M_t \BF M_t^\prime ~\leq~ \zeta^2 \, \kappa_0^2 \, \sigma_0^2 \, (\log t) \BF C_t^{-1}  ~\big|~ \BF Z = \BF z \right\}\\
	&\leq&   t^{r \, \kappa_0^2  } \; e^{ -\, \zeta^2 / 4 }~,
\eex
where the last inequality follows from Lemma \ref{lemma:positive-definite-matrix-tail}. 
\end{proof}

\subsection{Bounds for General Compact Sets of Arms: Proof of Theorem \ref{theorem:UE-regret-analysis}} 
\label{appendix:proof-UE-regret-analysis}


The proof of  Theorem \ref{theorem:UE-regret-analysis} makes
use of a number of auxiliary results.   The first result provides a motivation for the choice of
the parameter $\alpha$ in Equation (\ref{eq:uncertainty-ellipsoid}) and our definition of
the uncertainty radius $R^\BF u_t$ in Equation (\ref{eq:uncertainty-radius}).
They are chosen to keep the probability of overestimating the reward of an arm by
more than $R^{\BF u}_t$ bounded by $1/t^2$.  This will limit the growth rate of the cumulative regret due to such overestimation.  

\begin{lem}[Large Deviation Inequalities for  the Uncertainty Radius] 
\label{lemma:uncertainty-radius-tail-bound}
Under Assumption \ref{assumption:basic}, for any arm $\BF u \in \CAL U_r$ and $
t \geq r$,
$$
\Pr \left\{ \BF u^\prime \left( \widehat{\BF Z}_t - \BF z \right) > R_t^\BF u 
~\big|~ \BF Z = \BF z \right\} ~\leq~  \frac{1}{t^2}~,
$$
and for any $\BF x \in \Real^r$,
$$
\Pr \left\{  \left(\BF U_{t+1} - \BF x \right)^\prime \left( \widehat{\BF Z}_t - \BF z \right) > 
	\alpha  \, \sqrt{\log t} \,
		\sqrt{ \min\{ r \log t \,,\, \abs{\CAL U_r} \}} \; \norm{\BF U_{t+1} - \BF x}_{\BF C_t} 
		~\big|~ \BF Z = \BF z \right\} ~\leq~  \frac{1}{t^2}~,
$$
where 
$\alpha = 4 \sigma_0 \kappa_0^2$.
\end{lem}
\begin{proof}
It suffices to establish the first inequality because the proof of the second one
is exactly the same. Let $\beta_t = 4 \sigma_0 \kappa_0^2\, \sqrt{\log t} \, \sqrt{ \min\{ r \log t \,,\, \abs{\CAL U_r} \}}$.
Recall from Equations (\ref{eq:uncertainty-ellipsoid})  and (\ref{eq:uncertainty-radius}) that
$R^{\BF u}_t  = \beta_t \norm{\BF u}_{\BF C_t}$.  By applying
Theorem \ref{thm:chernoff-inq-for-UE-finite-arms} (with $\zeta = 
4 \kappa_0^2\, \sqrt{\log t} \, \sqrt{ \min\{ r \log t \,,\, \abs{\CAL U_r} \}}$)
and Theorem \ref{thm:chernoff-inq-for-UE-infinite-arms} (with 
$\zeta = 4 \kappa_0 \, \sqrt{ \min\{ r \log t \,,\, \abs{\CAL U_r} \}}$), we obtain
$$
 \Pr \left\{ \BF u^\prime \left( \widehat{\BF Z}_t - \BF z \right) > R^{\BF u}_t  
\,\big|\, \BF Z = \BF z \right\} 
~\leq~
\min \left\{   t^{5\abs{\CAL U_r}} e^{ - 8 \kappa_0^4 (\log t) \min\{ r \log t \,,\, \abs{\CAL U_r} \} }
 ~,~ t^{r \kappa_0^2} e^{-\,  4\,  \kappa_0^2 \min\{ r \log t \,,\, \abs{\CAL U_r} \}  } \right\}~.
$$
There are two cases to consider: $ r \log t > \abs{\CAL U_r}$ and $r \log t \leq \abs{\CAL U_r}$.
Suppose that $r \log t > \abs{\CAL U_r}$.  Then,
$$
\Pr \left\{ \BF u^\prime \left( \widehat{\BF Z}_t - \BF z \right) \;>\; R_t^\BF u 
\,\big|\, \BF Z = \BF z \right\} 
\leq
 t^{5\abs{\CAL U_r}} e^{ - 8 \kappa_0^4 (\log t) \min\{ r \log t \,,\, \abs{\CAL U_r} \} } 
=   t^{5\abs{\CAL U_r}} e^{ - 8 \kappa_0^4 (\log t) \abs{\CAL U_r} }
 =  \frac{ t^{5 \abs{\CAL U_r}}}{  t^{ 8 \kappa_0^4 \abs{\CAL U_r}}}  \leq \frac{1}{t^2}~, 
$$
where the last inequality follows from the fact that $\left( 8 \kappa_0^4  - 5 \right) \abs{\CAL U_r} \geq 2$.
In the second case where $r \log t \leq \abs{\CAL U_r}$, we have that
$$
\Pr \left\{ \BF u^\prime \left( \widehat{\BF Z}_t - \BF z \right) \;>\; R_t^\BF u 
\,\big|\, \BF Z = \BF z \right\} 
\leq
 t^{r \kappa_0^2} e^{-\,  4\,  \kappa_0^2 \min\{ r \log t, \abs{\CAL U_r} \}  } 
 =   t^{r \kappa_0^2} e^{-\,  4\,  \kappa_0^2 r \log t   }
= \frac{ t^{r \kappa_0^2}}{t^{4 r \kappa_0^2}} = \frac{1}{ t^{3 r \kappa_0^2}} \leq \frac{1}{t^2}~, 
$$
where the last inequality follows from the fact that $3 r \kappa_0^2 \geq 2$.  Since the probability
is bounded by $1/t^2$ in both cases, this gives the desired result.
\end{proof}

For any $t \geq 1$, let the random variable  $Q_t(\BF z)$
denote the instantaneous regret in period $t$ given that $\BF Z = \BF z$, that is,
\be
	Q_t (\BF z) &=&  \max_{\BF v \in \CAL U_r} \BF v^\prime \BF z  - \BF U_{t}^\prime \BF z~.
	\label{eq:instantaneous-regret}
\ee
Lemma \ref{lemma:uncertainty-radius-tail-bound}
shows that the probability of a large estimation error in period $t$ is at most $O \left(1/t^2\right)$.  Consequently,
as shown in the following lemma,  the probability of having a large instantaneous regret in period $t$ is also small. 

\begin{lem}[Instantaneous Regret Bound] \label{lemma:tail-bound-for-instantaneous-regret} Under Assumption \ref{assumption:basic}, for all $t \geq r$ and $\BF z \in \Real^r$,
$$
	\Pr \left\{ Q_{t+1}(\BF z) >  2 \alpha \, \sqrt{\log t} \sqrt{ \min \left\{ r \log t, \abs{\CAL U_r} \right\}}\, \norm { \BF U_{t+1}}_{\BF C_t}  ~\big|~ \BF Z = \BF z \right\}
	\leq  \frac{1}{t^2}~.
$$
\end{lem}
\begin{proof}
Let $\BF z \in \Real^r$ be given and let $\BF w$ denote an optimal arm, that is, $\max_{\BF v \in \CAL U_r} \BF v^\prime \BF z 
= \BF w^\prime \BF z$.   To facilitate our discussion, let $\beta_t =
\alpha \, \sqrt{\log t} \sqrt{ \min \left\{ r \log t, \abs{\CAL U_r} \right\}}$.  
Then,  it follows from the definition of the uncertainty radius in Equation (\ref{eq:uncertainty-radius}) and the 
definition of  $\BF U_{t+1}$ in Equation (\ref{eq:UE-decision}) that 
$$
		\BF U_{t+1}^\prime \widehat{\BF Z}_t +  \beta_t \, \norm{ \BF U_{t+1}}_{\BF C_t}
	\geq \BF w^\prime \widehat{\BF Z}_t +  \beta_t \, \norm{ \BF w}_{\BF C_t}~,
$$
which implies that
\bex
	 \beta_t \, \norm{ \BF U_{t+1}}_{\BF C_t} 
	&\geq& \left(\BF w -\BF U_{t+1}\right)^\prime \widehat{\BF Z}_t +  \beta_t \, \norm{ \BF w}_{\BF C_t}\\
	   &=& \left(\BF w-\BF U_{t+1}\right)^\prime \BF z + \left(\BF w-\BF U_{t+1}\right)^\prime 
	   \left( \widehat{\BF Z}_t - \BF z \right) +  \beta_t \, \norm{ \BF w}_{\BF C_t}\\
	   &=& Q_{t+1}(\BF z) + \left(\BF w-\BF U_{t+1}\right)^\prime \left( \widehat{\BF Z}_t - \BF z \right)  		+  \beta_t \, \norm{ \BF w}_{\BF C_t}~.
\eex

Suppose that the event $Q_{t+1}(\BF z) > 2  \beta_t \, \norm { \BF U_{t+1}}_{\BF C_t}$
occurs. Then, it follows that
\bex
	 \beta_t \, \norm{ \BF U_{t+1}}_{\BF C_t}
	   &>& 2\beta_t \, \norm { \BF U_{t+1}}_{\BF C_t}  + \left(\BF w-\BF U_{t+1}\right)^\prime \left( \widehat{\BF Z}_t - \BF z \right) +  \beta_t \,\norm{ \BF w}_{\BF C_t}~
\eex
which implies that
$
	\left(\BF U_{t+1} - \BF w \right)^\prime\left( \widehat{\BF Z}_t - \BF z \right) >\beta_t \, \left(  \norm { \BF U_{t+1}}_{\BF C_t}  + \norm{ \BF w}_{\BF C_t}  \right)
	   \geq \beta_t \,  \norm{ \BF U_{t+1} - \BF w  }_{\BF C_t}~.
$
Thus,
\bex
	\lefteqn{ \Pr \left\{ Q_{t+1}(\BF z) >  2 \beta_t \, \norm { \BF U_{t+1}}_{\BF C_t}  ~\big|~ \BF Z = \BF z \right\}}\\
	&\leq& \Pr \left\{ 	\left(\BF U_{t+1} - \BF w \right)^\prime \left( \widehat{\BF Z}_t - \BF z \right) 
	  ~>~\beta_t \, \norm{ \BF U_{t+1} - \BF w  }_{\BF C_t}  ~\big|~ \BF Z = \BF z \right\} \leq \frac{1}{t^2}~,
\eex
the last inequality follows from Lemma \ref{lemma:uncertainty-radius-tail-bound}.
\end{proof}

Lemma \ref{lemma:tail-bound-for-instantaneous-regret} suggests the following
approach for bounding the cumulative regret over $T$ periods.  In the first $r$ periods 
(during the initialization), we incur
a regret of $O(r)$.  For each time period between $r+1$ and $T$, we consider the 
two  cases:
1) where
the instantaneous regret is large with 
$Q_{t+1}(\BF z) >  2 \alpha \, \sqrt{\log t} \sqrt{ \min \left\{ r \log t, \abs{\CAL U_r} \right\}}\, \norm { \BF U_{t+1}}_{\BF C_t}$;
and,  2) the instantaneous regret is small.
By the above lemma, the contribution to the cumulative regret from the first case 
is bounded above by $O\left( \sum_t 1/t^2 \right)$, which is finite.  In the second case, we 
have a simple upper bound of $2 \alpha \, \sqrt{r} \, \left( \log t \right)  \norm { \BF U_{t+1}}_{\BF C_t}$ for the instantaneous
regret.  This argument leads to the following bound on the cumulative
regret over $T$ periods.  

\begin{lem}[Regret Decomposition]  \label{lem:regret-decomposition}  Under Assumption \ref{assumption:basic},
for all  $T \geq r+1$ and $\BF z \in \Real^r$,
$$
	{\rm Regret} \left ( \BF z, T, \textsc{UE} \right) \leq 2 \; \bar{u} ( r  + 2 ) \norm{\BF z} +  
	 2 \alpha \, \sqrt{r} \, \left( \log T \right) \sqrt{ T } \; \E \left[ \sqrt{ \sum_{t=r}^{T-1} \norm{ \BF U_{t+1}}_{\BF C_t}^2 } ~\Bigg|~ \BF Z = \BF z \right]~.
$$
\end{lem}
\begin{proof}
Let $\BF z \in \Real^r$ be given.
By the Cauchy-Schwarz Inequality and Assumption \ref{assumption:basic}(b), we have the following upper bound
on the instantaneous regret for all $t$ and $\BF z \in \Real^r$:
$ Q_t (\BF z) =  \max_{\BF v \in \CAL U_r} \left( \BF v - \BF U_{t} \right)^\prime \BF z
	\leq 2 \bar{u} \norm{\BF z}$.
Therefore,
$$
	\Regret \left ( \BF z, T, \textsc{UE} \right)
		\leq  2 \, \bar{u} \, r \norm{\BF z} +   \E \left[ \sum_{t=r}^{T-1}  Q_{t+1}(\BF z) ~\bigg|~ \BF Z = \BF z \right]~.
$$

For any $t \geq r$,
let the indicator random variable $G_{t+1}(\BF z)$ be defined by:
$$
	G_{t+1}(\BF z) =  \indicator \left[ Q_{t+1}(\BF z)   \leq 2 \alpha \, \sqrt{\log t} \sqrt{ \min \left\{ r \log t, \abs{\CAL U_r} \right\}}\,  \norm { \BF U_{t+1}}_{\BF C_t} \right]~.
$$
The contribution to the expected instantaneous regret 
$\E \left[ Q_{t+1}(\BF z) ~\big|~ \BF Z = \BF z \right]$ 
comes from two cases: 1) when $G_{t+1}(\BF z)  = 0$ and 2) when $G_{t+1}(\BF z) = 1$.
We will upper bound each of these two contributions separately.  In the first case, 
we know from Lemma \ref{lemma:tail-bound-for-instantaneous-regret} that 
$\Pr \left\{ G_{t+1}(\BF z)  = 0 ~\big|~ \BF Z = \BF z \right\}  = \Pr \left\{ Q_{t+1}(\BF z) >  2 \alpha \, \sqrt{\log t} \sqrt{ \min \left\{ r \log t, \abs{\CAL U_r} \right\}}\,   \norm { \BF U_{t+1}}_{\BF C_t}  ~\big|~ \BF Z = \BF z \right\} \leq 1/t^2$.  Since $\sum_{t=1}^\infty 1/t^2 \leq 2$, we have that
$$
\E \left[ \sum_{t=r}^{T-1}  \left(1 - G_{t+1}(\BF z) \right) Q_{t+1}(\BF z) ~\bigg|~ \BF Z = \BF z \right]
\leq 2 \bar{u} \norm{\BF z} \sum_{t=r}^{T-1} \Pr \left\{ G_{t+1}(\BF z)  = 0 ~\Big|~ \BF Z = \BF z \right\}
\leq 4 \bar{u} \norm{\BF z}~.
$$
On the other hand, when $G_{t+1}(\BF z) = 1$, we have that 
$Q_{t+1}(\BF z) \leq 2 \alpha \sqrt{r} (\log t) \,  \norm { \BF U_{t+1}}_{\BF C_t}$. This implies that, 
with probability one,
\bex
  \sum_{t=r}^{T-1} G_{t+1}(\BF z) Q_{t+1}(\BF z) 
		&\leq& 
	2  \, \alpha \, \sqrt{r} \, \sum_{t=r}^{T-1} \left( \log t \right) \norm { \BF U_{t+1}}_{\BF C_t} 
	 	~\leq~ 
	 2 \, \alpha \, \sqrt{r} \sqrt{ \sum_{t=r}^{T-1}     \log^2 t } \times \sqrt{ \sum_{t=r}^{T-1} \norm { \BF U_{t+1}}_{\BF C_t}^2}
	   \\
	  &\leq& 
	 2  \alpha \, \sqrt{r} \, \left( \log T \right) \sqrt{ T } \sqrt{ \sum_{t=r}^{T-1} \norm{ \BF U_{t+1}}_{\BF C_t}^2  }~,
\eex
where  we use the Cauchy-Schwarz Inequality in the second inequality and the 
final inequality follows from the fact that
$
	\sqrt{ \sum_{t=r}^{T-1} \log^2 t }  \, \leq \, \sqrt{ \sum_{t=1}^{T-1} \log^2 T}  \leq \left(\log T \right) \sqrt{T} ~.
$
Putting the two cases together gives the desired upper bound because
$$
\E \left[ \sum_{t=r}^{T-1}  Q_{t+1}(\BF z) ~\bigg|~ \BF Z = \BF z \right] 
\leq 4 \bar{u} \norm{\BF z} + 	 2  \alpha \, \sqrt{r} \, \left( \log T \right) \sqrt{ T} ~ \E \left[ \sqrt{ \sum_{t=r}^{T-1} \norm{ \BF U_{t+1}}_{\BF C_t}^2  } 
~\bigg|~ \BF Z = \BF z \right]~.
$$
\end{proof}

The eigenvectors of the  matrix $\BF C_t = \left( \sum_{s=1}^t \BF U_s \BF U_s^\prime \right)^{-1}$ reflect the directions of the arms
that are chosen during the first $t$ periods.  The corresponding eigenvalues then measure the frequency
with which these directions are explored.  Frequently explored directions will have
small eigenvalues, while the eigenvalues for unexplored directions will be large.
Thus, the  weighted norm $\norm{ \BF U_{t+1}}_{\BF C_t}$ has two interpretations. 
First, it measures the size of the regret in period $t+1$.  In addition, since 
$\norm{ \BF U_{t+1}}_{\BF C_t}^2$ is a linear combination of the eigenvalues of $\BF C_t$,
it also reflects the amount of exploration in period $t+1$ in the unexplored directions.  

The above interpretation suggests that if we incur large regrets in the past (equivalently, we have done a lot of exploration), then
the current regret should be small.  Our intuition is confirmed in the following lemma
that establishes a recursive relationship between the weighted norm $\norm{ \BF U_{t+1}}_{\BF C_t}$
in period $t+1$ and the norms in the preceding periods.

\begin{lem}[Large Past Regrets Imply Small Current Regret]  \label{lemma:weghted-norm-constraint}
Under Assumption \ref{assumption:basic},  for all $t \geq r$ and $\BF z \in \Real^r$, with probability one,
$$
	0 \leq \norm{ \BF U_{t+1} }_{\BF C_t}^2 \leq  \frac{\bar{u}^2}{\lambda_0}
	\qquad \textrm{and} \qquad
		\norm{ \BF U_{t+1} }_{\BF C_t}^2
	\leq  \frac{   \left\{ (\bar{u}^2 / \lambda_0) \cdot (t+1)\right\}^r}
	{\prod_{s=r}^{t-1} \left( 1 + \norm{ \BF U_{{s+1}}}_{\BF C_{s}}^2 \right)}~~.
$$
\end{lem}
\begin{proof}
For any $t \geq r$,
let $\BF \Upsilon_t = \left( \BF C_t  \right)^{-1} = \sum_{s=1}^t \BF U_{s} \BF U_{s}^\prime$. By the Rayleigh Principle,
$$
	\norm { \BF U_{t+1}}_{\BF C_t}^2  = \BF U_{t+1}^\prime \BF C_t \BF U_{t+1}
	\leq \lambda_{\max} \left( \BF C_t \right)  \norm{ \BF U_{t+1}}^2 = \frac{\norm{ \BF U_{t+1}}^2}{\lambda_{\min} \left( \BF \Upsilon_t \right)} \leq \frac{\bar{u}^2}{\lambda_0}~,
$$
where the last inequality follows from the definition of $\bar{u}$  and the fact that 
$$
	\lambda_{\min} \left( \BF \Upsilon_t \right) = \lambda_{\min} \left( \BF \Upsilon_r + \sum_{s=r+1}^t \BF U_{s} \BF U_{s}^\prime 
	\right) 
	\geq \lambda_{\min} \left( \BF \Upsilon_r \right) \geq \lambda_0~,
$$
where the last equality follows from the fact that $\BF \Upsilon_r = \sum_{k=1}^r {\BF b}_k {\BF b}_k^\prime$ where
the vectors $\BF b_1, \ldots, \BF b_r$ are given in Assumption \ref{assumption:basic}(b).
This proves the claimed upper bound on $\norm { \BF U_{t+1}}_{\BF C_t}^2$.

We will now establish the inequality that relates $\norm { \BF U_{t+1}}_{\BF C_t}^2$
to $\norm { \BF U_{{s+1}}}_{\BF C_s}^2$ for $s < t$.    Note that
\be
	\norm { \BF U_{t+1}}_{\BF C_t}^2  
	&=& \BF U_{t+1}^\prime \BF C_t \BF U_{t+1} 
	\leq 1 +  \BF U_{t+1}^\prime \BF C_{t} \BF U_{t+1} \nonumber\\
	&=& \frac{ \det \left( \BF \Upsilon_t \right)\cdot \left( 1 +  \BF U_{t+1}^\prime \BF C_{t} \BF U_{t+1} \right) }{ \det \left( \BF \Upsilon_t \right) } 
	= \frac{ \det \left( \BF \Upsilon_{t} \  + \BF U_{t+1}  \BF U_{t+1}^\prime  \right)}{ \det \left( \BF \Upsilon_t \right) }
	= \frac{ \det \left( \BF \Upsilon_{t+1}  \right)}{ \det \left( \BF \Upsilon_t \right)}~, 
	\label{eq:matrix-determinant-lemma}
\ee
where the second to last equality follows the  matrix
determinant lemma.

We will now establish bounds on the determinants $\det \left( \BF \Upsilon_{t+1}  \right)$ and  $\det \left( \BF \Upsilon_t \right)$.
Note that 
$$
\lambda_{\max} \left( \BF \Upsilon_{t+1} \right) \leq \tr \left( \BF \Upsilon_{t+1} \right)
=  \sum_{s=1}^{t+1} \tr \left( \BF U_{s} \BF U_{s}^\prime \right) =   \sum_{s=1}^{t+1} \norm{ \BF U_{s}}^2 
\leq (t+1) \bar{u}^2~,
$$
where the last inequality follows from the definition of $\bar{u}$ . Therefore,
$\det \left( \BF \Upsilon_{t+1}  \right) \leq \left[ \lambda_{\max} \left( \BF \Upsilon_{t+1} \right) \right]^r \leq (t+1)^r \bar{u}^{2r}$.
Moreover, using Equation (\ref{eq:matrix-determinant-lemma}) repeatedly,
we obtain
$$	
	\det \left( \BF \Upsilon_{t}  \right) 
	= \det(\BF \Upsilon_r)  \prod_{s=r}^{t-1} \left( 1 + \norm{ \BF U_{{s+1}}}_{\BF C_{s}}^2 \right)
		\geq  \lambda_0^r ~ \prod_{s=r}^{t-1} \left( 1 + \norm{ \BF U_{{s+1}}}_{\BF C_{s}}^2 \right)~,
$$
where the last inequality follows from the fact that
$\BF \Upsilon_r = \sum_{k=1}^r \BF b_k \BF b_k^\prime$ 
and $\det \left( \BF \Upsilon_r \right)
\geq \left[ \lambda_{\min} \left( \BF \Upsilon_r \right) \right]^r \geq \lambda_0^r$,
where the vectors $\BF b_1, \ldots, \BF b_r$ and the parameter $\lambda_0$
are  defined in Assumption \ref{assumption:basic}(b).

Putting everything together, we have that
$$
	\norm { \BF U_{t+1}}_{\BF C_t}^2  
	\leq \frac{ \det \left( \BF \Upsilon_{t+1}  \right)}{ \det \left( \BF \Upsilon_t \right)}
		\leq  \frac{ (t+1)^r \bar{u}^{2r}}{ \lambda_0^r \prod_{s=r}^{t-1} \left( 1 + \norm{ \BF U_{{s+1}}}_{\BF C_{s}}^2 \right)}
	=  \frac{   \left\{ (\bar{u}^2 / \lambda_0) \cdot (t+1)\right\}^r}{\prod_{s=r}^{t-1} \left( 1 + \norm{ \BF U_{{s+1}}}_{\BF C_{s}}^2 \right)}~,
$$
which is the desired result.
\end{proof}

The above result shows that if the  weighted norms in the preceding periods,
as measured by ${\prod_{s=r}^{t-1} \left( 1 + \norm{ \BF U_{{s+1}}}_{\BF C_{s}}^2 \right)}$, are large,
then the weighted norm  in the current period $t+1$ will be small.  Moreover, since the weighted
norm in the current period depends on the {\em product} of the norms  in the past,
we hope that the growth rate of the sum  $\sum_{t=r}^{T-1} \norm{ \BF U_{t+1}}_{\BF C_t}^2 $
should be small.
To formalize our conjecture, we introduce a related optimization problem.
For any $c \geq 0$ and $t \geq 1$, let $V^*(c,t)$ be defined by:
\bex
	V^*(c,t)   &=& \parbox{0.4in}{max}  \sum_{s=1}^t  y_s  \\
			&&\parbox{0.4in}{s.t.} 0 \leq y_s \leq c  \quad \textrm{and} \quad y_s
	\leq  \frac{   \left\{ c \cdot  (r+s)\right\}^r}
	{\prod_{q=1}^{s-1} \left( 1 + y_q \right)}~, \qquad s = 1, 2, \ldots, t ~,
\eex
where we define $\prod_{q=1}^0(1+y_q) = 1$. 
The following lemma gives an upper bound   in terms of
the function $V^*$.
\begin{lem}[Bounds on the 
Growth Rate of $\norm{\BF U_{t+1}}_{\BF C_t}^2$] \label{lem:optimization-connnection}  Under Assumption \ref{assumption:basic},
for any $T \geq r+1$ and $\BF z \in \Real^r$, with probability one,
$\sum_{t=r}^{T-1} \norm{ \BF U_{t+1}}_{\BF C_t}^2    \leq V^* \left( \bar{u}^2 / \lambda_0 ~,~  T-r  \right)$.
\end{lem}
\begin{proof}
For all $s \geq 1$, let $y_s = \norm{ \BF U_{r+s} }_{\BF C_{r+s-1}}^2$. Then, 
$\sum_{t=r}^{T-1} \norm{ \BF U_{t+1}}_{\BF C_t}^2 = \sum_{s=1}^{T-r} y_s$.  Let
$c_0 = \bar{u}^2 / \lambda_0$.  It follows from Lemma \ref{lemma:weghted-norm-constraint} that for all $s$, with probability one,
$0 \leq y_s \leq  c_0$
and
$ y_s \leq   \left\{ c_0 \,  (r+s)\right\}^r / \prod_{q=1}^{s-1} \left( 1 + y_q \right)$.
Therefore, we have  $\sum_{t=r}^{T-1} \norm{ \BF U_{t+1}}_{\BF C_t}^2    \leq V^* \left( c_0 ~,~  T-r  \right)$.
\end{proof}

It follows from Lemma \ref{lem:optimization-connnection} that it suffices to develop 
an upper bound on $V^*(c,t)$.  This result is given in the following lemma.

\begin{lem}[Optimization Bound] \label{lem:optimization-solution} For all $c \geq 0$, and $t \geq 1$,
$$
	V^*\left( c, t \right) \leq  2 \, c_0 \, \left( r\log c_0 +(r+1)\log(r+t+1)\right)~,
$$
where $c_0 = \max\{1, c\}$.
\end{lem}
\begin{proof}
Any feasible solution $\{y_s : s = 1, \ldots, t\}$ for the problem defining $V^*(c,t)$, also satisfies the constraints
$$
0\leq \frac{y_s}{2c_0}\leq \frac{y_s}{c_0} \leq 1\qquad {\rm and}\qquad
\frac{y_s}{2c_0}\leq 
\frac{\{ c_0\cdot(r+s)\}^r}
{\prod_{q=1}^{s-1}(1+(y_q/c_0))} \leq \{ c_0\cdot(r+s)\}^r e^{ - \sum_{q=1}^{s-1}  y_q / (2 c_0) }~,
$$
where the last inequality follows from the fact that  for any $a\in[0,1]$, we have $1+a\geq e^{a/2}$.
Thus, by letting $a_s=y_s/2c_0$, we obtain $V^*(c,t)\leq 2c_0 W^*(c_0,t)$, where
$W^*(c_0,t)$ is the maximum possible value of $\sum_{s=1}^t a_s$, subject to
$$
0\leq a_s\leq 1
\qquad {\rm and}\qquad
a_s\leq 
\{ c_0\cdot(r+s)\}^r
e^{- \sum_{q=1}^{s-1}a_s}~.
$$

Let us introduce a continuous-time variable $\tau$, and define
$a(\tau)=a_s$, for $\tau\in[s-1,s)$.
Let $b(\tau)=\int_0^{\tau} a(\tau')\, d\tau'$, and note that
$b(s)=\sum_{q=1}^s a_q$.
For any $\tau\in[s-1,s)$, we have
$${\dot b}(\tau)= a_s
\leq
\{ c_0\cdot(r+s)\}^r
e^{- \sum_{q=1}^{s-1}a_s}
\leq
\{ c_0\cdot(r+\tau+1)\}^r
e^{a_s - \sum_{q=1}^{s}a_q}
\leq
\{ c_0\cdot(r+\tau+1)\}^r  e^{-b(\tau) +1 }.$$
Let $d(\tau)=e^{b(\tau)}$. Then, for any $\tau \geq 0$,
$$
{\dot d}(\tau)= d(\tau) {\dot b(\tau)}
\leq e^{b(\tau)}
\{ c_0\cdot(r+\tau+1)\}^r e^{-b(\tau)+1}
=
\{ c_0\cdot(r+\tau+1)\}^r e.$$
By integrating both sides, we obtain
$ d(t)\leq \frac{e \, c_0^r \, (r+t+1)^{r+1}}{r+1}  $ for all $t \geq 0$.
Since $e / (r+1) \leq 1$ because $r \geq 2$, taking logarithms, we obtain
$$
\sum_{q=1}^t a_s =b(t) =\log d(t) \leq r \log c_0 +(r+1)\log(r+t+1).
$$
The right-hand side above is therefore an upper bound on $W^*(c_0,t)$,
which leads to the upper bound on $V^*(c,t)$, which gives the desired result.
\end{proof}

Finally, here is the proof of Theorem \ref{theorem:UE-regret-analysis}.

\begin{proof} It suffices to establish the regret bound because the risk bound
follows immediately from taking the expectation.
Let $A_0 = \max\{ 1, \bar{u}^2 / \lambda_0\}$.
It follows from Lemmas \ref{lem:regret-decomposition}, \ref{lem:optimization-connnection},
and \ref{lem:optimization-solution}  that
\bex
	\Regret \left ( \BF z, T, \textsc{UE} \right) &\leq&
2 \; \bar{u} ( r  + 2 ) \norm{\BF z} +  
	 2 { \alpha \, \sqrt{r} \,} \left( \log T \right) \sqrt{ T } \; \E \left[ \sqrt{ \sum_{t=r}^{T-1} \norm{ \BF U_{t+1}}_{\BF C_t}^2 } ~\Bigg|~ \BF Z = \BF z \right]\\
	 	 &\leq&
2 \; \bar{u} ( r  + 2 ) \norm{\BF z} +  
	 2 { \alpha \, \sqrt{r} \,} \left( \log T \right)  \sqrt{ T } \;  \sqrt{ V^* \left( \bar{u}^2 / \lambda_0 ~,~  T-r  \right)  }\\
	 &\leq&
2 \; \bar{u} ( r  + 2 ) \norm{\BF z} +  
	 2 { \alpha \, \sqrt{r} \,} \left( \log T \right)  \sqrt{ T } \;  \sqrt{ 
			2 A_0  \left\{ r \log A_0 + (r+1) \log (T+1)  \right\}   }\\
	 &\leq& a_4 \, r \norm{\BF z} + a_5  \, r \, \sqrt{T} \log^{3/2} T~,
\eex
for some positive constants $a_4$ and $a_5$ that depend only on $\sigma_0$, $\bar{u}$,  and $\lambda_0$.
\end{proof}

\subsection{Bounds for Finitely Many Arms: Proof of Theorem \ref{thm:regret-and-risk-small-arms}}
\label{appendix:proof-regret-and-risk-small-arms}

Recall that for any $\BF z \in \Real^r$ and $\BF u \in \CAL U_r$, 
$N^\BF u(\BF z, T)$ is the number of times that arm $\BF u$ has been chosen
during the first $T$ periods.  To complete the proof of 
Theorem \ref{thm:regret-and-risk-small-arms}, it suffices to show that
$$
	\E \left[ N^\BF u(\BF z, T) \mid \BF Z = \BF z\right] \leq 6 
		+\frac{4 \alpha^2  \abs{\CAL U_r}  \log T }{ \left( \Delta^{\BF u} \left( \BF z \right) \right)^2} ~.
$$
Let us fix some $\BF z \in \Real^r$ and $\BF u \in \CAL U_r$. 
 Since $N^{\BF u}(\BF z, t)$ is nondecreasing in $t$, we can show that
for any positive integer $\theta$,
$ N^{\BF u}(\BF z, T) 
~\leq~  \theta + \sum_{t=r}^{T-1} \indicator_{\left\{ \BF U_{t+1} = \BF u \textrm{ and } N^{\BF u}(\BF z, t) ~\geq~ \theta \right\}}$.
Suppose that $\BF w$ is the optimal arm, that is, $\max_{\BF v \in \CAL U_r} \BF v^\prime \BF z = \BF w^\prime \BF z$.
Then, we have that
$$
\indicator_{\left\{ \BF U_{t+1} = \BF u \right\}}
~\leq~
\indicator_{\left\{  \BF u^\prime \widehat{ \BF Z}_t + R_t^{\BF u} ~\geq~ \BF w^\prime \widehat{ \BF Z}_t + R_t^{\BF w} \right\}}
~\leq~ \indicator_{\left\{  \BF u^\prime \left( \widehat{ \BF Z}_t - \BF z \right) ~>~ R_t^{\BF u} \right\}} 
+ \indicator_{\left\{  \BF w^\prime \left( \widehat{ \BF Z}_t - \BF z \right) ~<~ -R_t^{\BF w} \right\}}
+ \indicator_{\left\{ \left( \BF w - \BF u \right)^\prime \BF z ~\leq~ 2 R_t^\BF u
 \right\}}~.
$$
Since $\Pr  \left\{  \BF u^\prime \left( \widehat{ \BF Z}_t - \BF z \right) > R_t^{\BF u} 
\mid \BF Z = \BF z \right\}$ and $\Pr \left\{  \BF w^\prime \left( \widehat{ \BF Z}_t - \BF z \right) < -R_t^{\BF w} \mid \BF Z = \BF z \right\}$ are bounded above by $1/t^2$ by Lemma \ref{lemma:uncertainty-radius-tail-bound}, we can show that
\bex
	\E \left[ N^{\BF u}(\BF z, T) \mid \BF Z = \BF z\right] 
		&\leq& \theta + 2 \sum_{t=1}^\infty \frac{1}{ t^2} + \sum_{t=r}^{T-1} \Pr \left\{ \left( \BF w - \BF u \right)^\prime \BF z ~\leq~ 2 R_t^{\BF u} 
			\textrm{ and } N^{\BF u}(\BF z, t) \geq \theta \mid \BF Z = \BF z  \right\}~,\\
	&\leq& 4 + \theta + \sum_{t=r}^{T-1} \Pr \left\{ \left( \BF w - \BF u \right)^\prime \BF z ~\leq~ 2 R_t^{\BF u} 
			\textrm{ and } N^{\BF u}(\BF z, t) \geq \theta \mid \BF Z = \BF z  \right\}~.
\eex

Let $\BF H = \sum_{\BF v  \in \CAL U_r : \BF v \ne \BF u} N^\BF v (\BF z, t) \BF v \BF v^\prime$. 
It follows from Equation (\ref{eq:OLS}) and 
the Sherman-Morrison Formula (see \citealp{ShermanMorrison:1950}) that
\bex
\BF C_t &=&  \left( \BF H + N^{\BF u}(\BF z, t) \BF u \BF u^\prime \right)^{-1}
		       = \BF H^{-1} - \frac{ N^{\BF u}(\BF z,t)  \BF H^{-1} \BF u \BF u^\prime \BF H^{-1}}
								{1 + N^{\BF u}(\BF z,t)  \BF u^\prime \BF H^{-1} \BF u},
\eex
which implies that
$$
\norm{\BF u}_{\BF C_t}^2  = \BF u^\prime \BF C_t \BF u
= \BF u^\prime \BF H^{-1} \BF u  -  \frac{ N^{\BF u}(\BF z,t)  \left( \BF u^\prime \BF H^{-1} \BF u \right)^2}
								{1 + N^{\BF u}(\BF z,t)  \BF u^\prime \BF H^{-1} \BF u} = \frac{  \BF u^\prime \BF H^{-1} \BF u }{1 + N^{\BF u}(\BF z,t)  \BF u^\prime \BF H^{-1} \BF u} 
\leq \frac{1}{N^{\BF u}(\BF z, t)}~,
$$
and therefore,
$
2 R_t^{\BF u} = 2 \alpha \,  \sqrt{ \log t} \sqrt{ \min\left\{ r  \log t , \abs{\CAL U_r} \right\}}  \norm { \BF u}_{\BF C_t} \leq \big( 2 \alpha \, \sqrt{\abs{\CAL U_r} \, \log t} \,\big) / \sqrt{ N^{\BF u}(\BF z,t)}$.

By setting $
\theta =  1+ \left \lceil  \frac{ 4 \alpha^2 \abs{\CAL U_r}  \log T }{\left( \Delta^{\BF u} \left( \BF z \right) \right)^2} \right \rceil~,
$
we conclude that
$2 R_t^{\BF u} < \Delta^{\BF u} \left( \BF z \right) = \left( \BF w - \BF u \right)^\prime \BF z $ whenever $N^{\BF u}(\BF z, t) \geq \theta$.  This implies that
$
	\indicator_{\left\{\left( \BF w - \BF u \right)^\prime \BF z  ~\leq~ 2 R_t^{\BF u} 
	\; \textrm{and} \;
	N^{\BF u}(\BF z, t) \geq \theta \right\}} = 0~
$, and we have that
$$
	\E \left[ N^{\BF u}(\BF z, T) \mid \BF Z = \BF z\right] 
	\leq 4 + 1 + \left \lceil  \frac{ 4 \alpha^2 \abs{\CAL U_r} \log T }{\left( \Delta^{\BF u} \left( \BF z \right) \right)^2} \right \rceil		
	\leq  6 + \frac{4 \alpha^2  \abs{\CAL U_r} \log T }{ \left( \Delta^{\BF u} \left( \BF z \right) \right)^2}  ~,
$$
which is the desired result.
\end{document}